\documentclass{article}

\usepackage{PRIMEarxiv}

\usepackage[utf8]{inputenc} 
\usepackage[T1]{fontenc}    
\usepackage{hyperref}       
\usepackage{url}            
\usepackage{booktabs}       
\usepackage{amsfonts}       
\usepackage{microtype}      
\usepackage{fancyhdr}       
\usepackage{graphicx}  
\usepackage[table,xcdraw]{xcolor}
\usepackage{colortbl}
\usepackage{multirow}
\usepackage{amsmath}
\usepackage{float}
\usepackage{csquotes}      
\usepackage{amsthm}
\usepackage{makecell} 
\usepackage[export]{adjustbox}
\newtheorem{definition}{Definition}[section]
\newtheorem{proposition}{Proposition}[section]
\newtheorem{theorem}{Theorem}[section]
\newcommand{\mbf}[1]{\mathbf{#1}}

\newcommand{\X}{\mathbf{X}}

\newcommand{\VO}{\mathbf{O}}
\newcommand{\vo}{\mathbf{o}}

\newcommand{\Y}{\hat{Y}}
\newcommand{\y}{\hat{y}}

\newcommand{\D}{\mathcal{D}}
\newcommand{\CG}{\mathcal{G}}

\newcommand{\Pa}[1]{\mathsf{PA}_{#1}}
\newcommand{\pa}[1]{\mathsf{pa}_{#1}}

\newcommand{\TCE}{\mathrm{TCE}}
\newcommand{\PE}{\mathrm{PE}}
\newcommand{\CE}{\mathrm{CE}}
\newcommand{\PCE}{\mathrm{PCE}}

\DeclareMathOperator*{\argmin}{\arg\!\min}

\title{The Fairness Field Guide: Perspectives from Social and Formal Sciences
}

\author{
  Alycia N. Carey, Xintao Wu \\
  Department of Computer Science and Computer Engineering \\
  University of Arkansas \\
  Fayetteville, Arkansas\\
  \texttt{\{ancarey, xintaowu\}@uark.edu} \\
}

\begin{document}
\maketitle

\begin{abstract}
Over the past several years, a slew of different methods to measure the fairness of a machine learning model have been proposed. However, despite the growing number of publications and implementations, there is still a critical lack of literature that explains the interplay of fair machine learning with the social sciences of philosophy, sociology, and law. We hope to remedy this issue by accumulating and expounding upon the thoughts and discussions of fair machine learning produced by both social and formal (specifically machine learning and statistics) sciences in this field guide. Specifically, in addition to giving the mathematical and algorithmic backgrounds of several popular statistical and causal-based fair machine learning methods, we explain the underlying philosophical and legal thoughts that support them. Further, we explore several criticisms of the current approaches to fair machine learning from sociological and philosophical viewpoints. It is our hope that this field guide will help fair machine learning practitioners better understand how their algorithms align with important humanistic values (such as fairness) and how we can, as a field, design methods and metrics to better serve oppressed and marginalized populaces.
\end{abstract}


\section{Introduction}

In 2019, Obermeyer et al. published a disconcerting article that illustrated how a widely used health risk-prediction tool, that is applied to roughly 200 million individuals in the U.S. per year, exhibited significant racial bias \cite{obermeyer_powers_vogeli_mullainathan_2019}. They found that this bias was due to the algorithm equating the predicted health-care cost of a patient with how sick they were likely to be. But, due to systemic racism issues, such as distrust of the health-care system and direct racial discrimination, Black individuals usually have more barriers to receiving proper health care. Therefore, they generally spend less on health-care costs per year \cite{10.1001/jama.2021.9937}. Obermeyer et al. found that only 17.7\% of patients that the algorithm assigned to receive extra care were Black, but if the bias in the system was corrected for, the percentage would increase drastically to 46.5\% \cite{ledford_2019}.

Numerous examples like the one above exist which depict machine learning algorithms as being biased to a particular marginalization class such as race or gender. For instance, women are less likely to be shown ads for high paid positions like CEO or CFO \cite{datta2015automated}, and facial recognition systems are more likely to be wrong when given a picture of a Black woman \cite{buolamwini_gender_2018}. Finding deficiencies such as these has spurred research into creating machine learning models that achieve fair outcomes.

Defining, implementing, and enforcing fairness in machine learning is, above all else, a sociotechnical\footnote{The field of Science and Technology Studies (STS) describes systems that consist of a combination of technical and social components as ``sociotechnical systems'' \cite{selbst2019abstraction}.} challenge. Machine learning systems may behave unfairly, or in a biased manner, for a multitude of reasons: e.g., due to social biases reflected in the datasets used in training, because of societal biases that are reflected (explicitly or implicitly) in the design process, and/or as a result of interaction with particular stakeholders or clients during run-time \cite{madaio_ai_nodate}. 

Without viewing fairness metrics from lenses of philosophy, sociology, and law, choosing and implementing a metric stays firmly technical, and does not consider societal impacts that could arise after deployment. In this instance, we are blindly hoping to align to our wanted societal principles. To solve this problem, and to help practitioners choose correct fairness metrics in an informed, societal-aware, manner, we develop the following field guide that depicts popular binary statistical- and causal-based fairness metrics through lenses of philosophy, sociology, and the law. We note that this field guide does not cover bias mitigation strategies and we refer interested readers to \cite{aif360-oct-2018, caton2020fairness, mehrabi2019survey} for more in depth discussions on this topic.

The rest of the field guide is as follows. We begin by explaining commonly (mis)used definitions in fair machine learning literature and introduce mathematical notions that will be used throughout the rest of the field guide in Section 2. In Section 3, we introduce important philosophical perspectives, such as Rawl's Equality of Opportunity (EOP), that serve as a foundation for many of the proposed fairness metrics. Next, in Section 4, we depict popular legal ideals that have a strong connection to fairness in machine learning. Sections 5 and 6 contain our analysis and discussion on popular statistical- and causal-based fairness metrics. In Section 7, we give several critiques from philosophy and sociology of the fair machine learning field as a whole. Finally, in Sections 8 and 9 we detail closely related work and our major conclusions, respectively. 

\section{Background}
Being a sociotechnical issue, the terminology used in fair machine learning literature can take on different meanings depending on the context in which the word is being used. This can potentially lead to confusion when selecting the right fairness metric to deploy in specific situations. For this reason, in Table 1 we concretely define the meaning of commonly used terms in fair machine learning publications in lenses of both philosophy and computer science. Additionally, in this section we provide definitions for important statistical terms as well as three overarching statistical categories that are commonly used to classify machine learning fairness metrics. Further, we lay the groundwork for the discussion on causal models in Section 6 by stating several important frameworks and definitions.

\subsection{Statistical Definitions}
Most proposed machine learning fairness metrics have groundings in statistical measures \cite{chouldechova_frontiers_2018}. For example, statistical parity (otherwise known as group fairness) depends on the measurement of raw positive classification rates; equalized odds depends on false positive and false negative rates; and predictive parity depends on true positive rates. The use of statistical measures is attractive because they are relatively simple to measure and definitions built using statistical measures can usually be achieved without having to make any assumptions on the underlying data distributions. Many of the common statistical measures used in machine learning fairness metrics are listed in Table \ref{tab:stat_def}.

It is important to note, however, that statistical groundings do not provide for individual level fairness, or even sub-group fairness for a marginalized class \cite{corbettdavies2018measure}. Instead, they provide meaningful guarantees to the ``average" member of a marginalized group. Additionally, many statistical measures directly oppose one another. For instance, it is impossible to satisfy both false positive rates, false negative rates, and positive predictive value across marginalized groups. This creates the direct consequence that many definitions of machine learning fairness cannot be satisfied at the same time.

This fact was firmly cemented in the work completed by Barocas, Hardt, and Narayanan \cite{barocas-hardt-narayanan}. In this work, they propose three representative fairness criteria -- independence, separation, and sufficiency -- that relate to many of the fair machine learning metrics that have been published. They capitalize on the fact that most proposed fairness criteria are properties of the joint distribution of a marginalization\footnote{What we call the marginalization attribute (or marginalized class) is often called the sensitive or protected attribute/class in fair machine learning literature. See Section 7.2 for why we choose to use marginalized instead of protected or sensitive.} attribute $S$ (e.g., race or gender), a target variable $Y$, and the classification (or in some cases probability score) $\hat{Y}$, which allowed them to create three distinct categories by forming conditional independence statements between the three random variables.

\subsubsection{Independence}
The first formal category, independence, only requires that the marginalization attribute, $S$ ($S = 0$ non-marginalized, $S = 1$ marginalized), is statistically independent of the classification outcome $\hat{Y}$, $\hat{Y} \perp S$. For the binary classification case, they note two different formulations: 

\begin{table}[H]
    \centering
    \begin{tabular}{rl}
        Exact: & $P[\hat{Y} = 1 \;|\; S = 0] = P[\hat{Y} = 1 \;|\; S = 1]$ \\
        & \\
        Relaxed: & $\frac{P[\hat{Y}=1 \;|\; S = 0]}{P[\hat{Y} = 1 \;|\; S = 1]} \geq 1 - \epsilon$
    \end{tabular}
\end{table}

\begin{table}[t!]\small
\centering
\setlength{\extrarowheight}{-2pt}
\addtolength{\extrarowheight}{\aboverulesep}
\addtolength{\extrarowheight}{\belowrulesep}
\setlength{\aboverulesep}{0pt}
\setlength{\belowrulesep}{0pt}
\caption{Definitions of popular terms from the view point of philosophy and computer science.}
\resizebox{\linewidth}{!}{%
\begin{tabular}{ccc} 
\toprule
\rowcolor[rgb]{0.592,0.792,0.792} \textbf{\textcolor[rgb]{0,0.502,0.502}{Term}} & \textbf{\textcolor[rgb]{0,0.502,0.502}{Philosophy}}                                                                                                                                                                                                                                                                                                               & \textbf{\textcolor[rgb]{0,0.502,0.502}{Computer Science}}                                                                                                                                                                                                                                                                                                                                                                             \\ 
\hline
Bias                                                                            & \begin{tabular}[c]{@{}c@{}}Prejudice in favor of or against one thing, \\ person, or group compared with another, \\ usually in a way considered to be unfair.\end{tabular}                                                                                                                                                                                       & \begin{tabular}[c]{@{}c@{}}Occurs when an algorithm produces results \\that are systemically prejudiced due to erroneous\\assumptions in the machine learning process. Normally\\due to data collection, sampling, and/or measurement~\\procedures.\\\\Bias vector:~an additional set of weights in a~neural\\network~that require no input,~giving importance to \\some of the~features in order to generalize better.\end{tabular}  \\ 
\hline
Oppression                                                                      & \begin{tabular}[c]{@{}c@{}}A social system of barriers that operate institutionally\\and interpersonally to disempower people because\\of their gender, race, class, sexuality, ethnicity, religion\\body size, ability, and/or nationality.\end{tabular}                                                                                        & -                                                                                                                                                                                                                                                                                                                                                                                                                                     \\ 
\hline
Ethics                                                                          & \begin{tabular}[c]{@{}c@{}}Involves systematizing, defending, and \\recommending concepts of right and wrong\\ behavior.\end{tabular}                                                                                                                                                                                                                             & \begin{tabular}[c]{@{}c@{}}Area of artificial intelligence research concerned \\with adding or ensuring moral behaviors of man-made \\machines that use artificial intelligence.\end{tabular}                                                                                                                                                                                                                                    \\ 
\hline
Justice                                                                         & \begin{tabular}[c]{@{}c@{}}Adherence to the standards agreed upon\\in society (based on laws).\end{tabular}                                                                                                                                                                                                                                                       & -                                                                                                                                                                                                                                                                                                                                                                                                                                     \\ 
\hline
Fairness                                                                        & \begin{tabular}[c]{@{}c@{}}A subjective principle of judgement of \\ whether a decision is morally right or wrong.\end{tabular}                                                                                                                                                                                                                                   & \begin{tabular}[c]{@{}c@{}}Recently established area of machine learning that \\studies how to ensure that biases in the data, and \\model inaccuracies, do not lead to models that treat \\individuals unfavorably on the basis of sensitive \\characteristics.\end{tabular}                                                                                                                                                         \\ 
\hline
Equality                                                                        & \begin{tabular}[c]{@{}c@{}}\textcolor[rgb]{0.102,0.102,0.102}{Signifies correspondence between a group}\\\textcolor[rgb]{0.102,0.102,0.102}{of different objects, persons, processes or}\\\textcolor[rgb]{0.102,0.102,0.102}{circumstances that have the same qualities }\\\textcolor[rgb]{0.102,0.102,0.102}{in at least one respect, but not all.}\end{tabular} & -                                                                                                                                                                                                                                                                                                                                                                                                                                     \\ 
\hline
~Discrimination~~                                                               & \begin{tabular}[c]{@{}c@{}}Unfavorable treatment of people due to the\\membership in certain demographic groups \\that are distinguished by attributes \\(supposedly) protected by law.\end{tabular}                                                                                                                                                                              & \begin{tabular}[c]{@{}c@{}}A source for unfairness in machine learning due\\to (intentional or unintentional) human prejudice \\and stereotyping based on the sensitive attributes.\\\\GAN Discriminator: t\textcolor[rgb]{0.122,0.122,0.141}{ries to distinguish real data }\\\textcolor[rgb]{0.122,0.122,0.141}{from the data created by the generation.}\end{tabular}                                                              \\ 
\hline
\begin{tabular}[c]{@{}c@{}}Marginalization\\ (Sensitive /Protected)\\Characteristic\end{tabular}               & \multicolumn{2}{c}{\begin{tabular}[c]{@{}c@{}}Those characteristics commonly referenced and reflected in non-discrimination law. E.g., race, ethnicity,\\ gender, religion, age, disability, sexual orientation, etc.\end{tabular}}                                                                                                                                                                                                                                                                                                                                                                                                                                                                                                                                                                       \\
\bottomrule
\end{tabular}
}
\end{table}
\vspace{-10pt}
When considering the event $\hat{Y}=1$ to be the positive outcome, this condition requires the acceptance rates to be the same across all groups. The relaxed version notes that the ratio between the acceptance rates of different groups needs to be greater than a threshold that is determined by a predefined slack term $\epsilon$. In many cases $\epsilon = .2$ in order to align with the four-fifths rule in disparate impact law (see Section 4.1).

Barocas, Hardt, and Narayanan also note that while independence aligns well with how humans reason about fairness, several draw-backs exists for fairness metrics that fall into this category (e.g., group fairness, disparate impact, conditional statistical parity, and overall accuracy equality) \cite{barocas-hardt-narayanan}. Namely, that they ignore any correlation between the sensitive attributes and the target variable $Y$ which constrains the construction of a perfect prediction model. Additionally, it enables laziness. In other words, it allows situations where qualified people are carefully selected for one group (e.g., non-marginalized), while random people are selected for the other (marginalized). Further, it allows the trade of false negatives for false positives, meaning that neither of these rates are considered more important, which is false in many circumstances \cite{barocas-nips-2017}.

\begin{table}[t!]\small
\centering
\setlength{\extrarowheight}{0pt}
\addtolength{\extrarowheight}{\aboverulesep}
\addtolength{\extrarowheight}{\belowrulesep}
\setlength{\aboverulesep}{0pt}
\setlength{\belowrulesep}{0pt}
\caption{Definitions for common statistical measures. $Y$ = actual, $\hat{Y}$ = predicted, 1 = positive, 0 = negative.}
\label{tab:stat_def}
\resizebox{\linewidth}{!}{%
\begin{tabular}{cccc} 
\toprule
\rowcolor[rgb]{0.592,0.792,0.792} \textbf{\textcolor[rgb]{0,0.502,0.502}{Measure}} & \textbf{\textcolor[rgb]{0,0.502,0.502}{Abrriv.}} & \textbf{\textcolor[rgb]{0,0.502,0.502}{Formula}} & \textbf{\textcolor[rgb]{0,0.502,0.502}{Definition}}                                                                                                           \\ 
\hline
True Positive                                                                      & TP                                               & $Y = 1 \cap \hat{Y} = 1$                         & \begin{tabular}[c]{@{}c@{}}A case when the predicted and actual outcome\\ are both in the positive class.\end{tabular}                                        \\ 
\hline
True Negative                                                                      & TN                                               & $Y = 0 \cap \hat{Y} = 0$                         & \begin{tabular}[c]{@{}c@{}}A case when the predicted and actual outcome\\ are both in the negative class.\end{tabular}                                        \\ 
\hline
False Positive                                                                     & FP                                               & $Y = 0 \cap \hat{Y} = 1$                         & \begin{tabular}[c]{@{}c@{}}A case predicted to be in the positive class when\\ the actual outcome belongs to the negative class.\end{tabular}                 \\ 
\hline
False Negative                                                                     & FN                                               & $Y = 1 \cap \hat{Y} = 0$                         & \begin{tabular}[c]{@{}c@{}}A case predicted to be in the negative class when\\ the actual outcome belongs to the positive class.\end{tabular}                 \\ 
\hline
\begin{tabular}[c]{@{}c@{}}Positive Predictive Value \\(Precision)\end{tabular}    & PPV                                              & $\frac{TP}{TP + FP}$                             & \begin{tabular}[c]{@{}c@{}}Fraction of positive cases correctly predicted to be\\ in the positive class out of all predicted positive cases.\end{tabular}     \\ 
\hline
False Discovery Rate                                                               & FDR                                              & $\frac{FP}{TP + FP}$                             & \begin{tabular}[c]{@{}c@{}}Fraction of negative cases incorrectly predicted to be\\ in the positive class out of all predicted positive cases.\end{tabular}   \\ 
\hline
False Omission Rate                                                                & FOR                                              & $\frac{FN}{TN + FN}$                             & \begin{tabular}[c]{@{}c@{}}Fraction of positive cases incorrectly predicted to be \\ in the negative class out of all predicted negative cases.\end{tabular}  \\ 
\hline
\begin{tabular}[c]{@{}c@{}}Negative Predictive\\ Value\end{tabular}                & NPV                                              & $\frac{TN}{TN + FN}$                             & \begin{tabular}[c]{@{}c@{}}Fraction of negative cases correctly predicted to be\\ in the negative class out of all negative cases.\end{tabular}               \\ 
\hline
\begin{tabular}[c]{@{}c@{}}True Positive Rate\\ (Sensitivity/ Recall)\end{tabular} & TPR                                              & $\frac{TP}{TP + FN}$                             & \begin{tabular}[c]{@{}c@{}}Fraction of positive cases correctly predicted to be\\ in the positive class out of all actual positive cases.\end{tabular}        \\ 
\hline
False Positive Rate                                                                & FPR                                              & $\frac{FP}{TN + FP}$                             & \begin{tabular}[c]{@{}c@{}}Fraction of negative cases incorrectly predicted to be\\ in the positive class out of all actual negative cases.\end{tabular}      \\ 
\hline
False Negative Rate                                                                & FNR                                              & $\frac{FN}{TP + FN}$                             & \begin{tabular}[c]{@{}c@{}}Fraction of positive cases incorrectly predicted to be\\ in the negative class out of all actual positive cases.\end{tabular}      \\ 
\hline
True Negative Rate                                                                 & TNR                                              & $\frac{TN}{TN + FP}$                             & \begin{tabular}[c]{@{}c@{}}Fraction of negative cases correctly predicted to be\\ in the negative class out of all actual negative cases.\end{tabular}        \\
\bottomrule
\end{tabular}
}
\end{table}

\subsubsection{Separation}
The second category Barocas, Hardt, and Narayanan propose is separation, which captures the idea that in many scenarios, the sensitive characteristic may be correlated with the target variable \cite{barocas-hardt-narayanan}. Specifically, the random variables satisfy separation if $\hat{Y} \perp S \;|\; Y$, or in other words, if $\hat{Y}$ is conditionally independent of $S$ when given $Y$. In the binary classification case, it is equivalent to requiring that all groups achieve the same true and false positive rates.
$$TP\;:\;P[\hat{Y}=1 \;|\; Y = 1 \cap S = 0] = P[\hat{Y}=1 \;|\; Y = 1 \cap S = 1]$$
$$FP\;:\;P[\hat{Y}=1 \;|\; Y = 0 \cap S = 0] = P[\hat{Y}=1 \;|\; Y = 0 \cap S = 1]$$
Additionally, this requirement can be relaxed to only require the same true positive rates or the same false positive rates. Fairness metrics that fall under both the equivalent and relaxed version of separation include: false positive error rate balance, false negative error rate balance, equalized odds, treatment equality, balance for the positive class, and balance for the negative class. 

\subsubsection{Sufficiency}
The final category, sufficiency, makes use of the idea that for the purpose of predicting $Y$, the value of $S$ doesn't need to be used when you are given $\hat{Y}$ \cite{barocas-nips-2017}. In other words, the classification $\hat{Y}$ already has made use of $S$ to predict the target $Y$. Specifically, the random variables satisfy sufficiency if $Y \perp S \;|\; \hat{Y}$. I.e., $Y$ is conditionally independent of $S$ given $\hat{Y}$. In the binary classification case, this is the same as requiring a parity of positive or negative predictive values across all groups $\hat{y} \in \hat{Y}= \{0, 1\}$: 
$$P[Y = 1 \;|\; \hat{Y} = \hat{y} \cap S = 0] = P[Y = 1 \;|\; \hat{Y} = \hat{y} \cap S = 1]$$

They note that it is common to assume that $\hat{Y}$ satisfies sufficiency if the marginalization attribute $S$ and the target variable $Y$ are clearly understood from the problem context. Some examples of fairness metrics that satisfy sufficiency include: predictive parity, conditional use accuracy, test fairness, and well calibration.

\subsubsection{Impossibility Results of Statistics-Based Metrics}
Although each of the statistics-based fairness metrics we introduce in Section 5 formalize an intuitive notion of fairness, the definitions are not mathematically compatible in general. In other words, some definitions of fairness cannot be enforced at the same time. These incompatibilities between the fairness definitions were first explored during public debate over a recidivism\footnote{Recidivism: the tendency of a convicted criminal to re-offend.} tool called COMPAS (Correctional Offender Management Profiling for Alternative Sanctions) \cite{mitchell2021}. While ProPublica proved that COMPAS does not satisfy false positive error rate balance (see Section 5.2.2) \cite{angwin-jeff_2016}, other researchers found that it did satisfy metrics such as predictive parity (see Section 5.2.1) and test fairness (see Section 5.3.1) \cite{dieterich_mendoza_brennan_2016, flores2016}.

The tension experienced here is due to impossibility results that govern the underlying statistics of the different fairness measures. This notion is backed by several research publications including \cite{barocas-hardt-narayanan} who explained how independence, separation, and sufficiency are mutually exclusive, and \cite{kleinberg2016inherent, chouldechova2017fair} who both showed that if a model satisfies balance for the negative class, balance for the positive class, and test fairness among marginalized and non-marginalized groups, then there must be equal base rates (which implies that the actual classification was independent of the group) or the model was 100\% accurate \cite{mitchell2021}. 

\subsection{Causal Frameworks}
\label{sec:cframework}
Here, we briefly introduce two fundamental causal frameworks which are required for causal-based fairness notions: the structural causal model (SCM) \cite{pearl2009causality} and the potential outcome framework \cite{imbens2015causal}. We refer interested readers to recent surveys, such as \cite{guo2020survey} and \cite{yao2021survey}, in the area of causal inference and causal discovery for a more in depth discussion. We emphasize that the two causal frameworks are logically equivalent, although they have different assumptions. In SCMs, we are able to study the causal effect of any variable along different paths as the complete causal graph is often assumed. On the other hand, in the potential outcome framework, we do not assume the availability of the causal graph and focus on estimating causal effects of the treatment variables.

Throughout the following discussion, and in Section \ref{sec:cfair} (which details the causal-based fairness notions), we use the following notation conventions. An uppercase letter denotes a variable, e.g., $X$; a bold uppercase letter denotes a set of variables, e.g., $\mathbf{X}$; and a lowercase letter denotes a value or a set of values of the corresponding variables, e.g., $x$ and $\mathbf{x}$.

\subsubsection{Structural Causal Model}
The structural causal model (SCM) was first proposed by Judea Pearl in \cite{pearl2009causality}. Pearl believed that by understanding the logic behind causal thinking, we would be able to emulate it on a computer to form more realistic artificial intelligence \cite{PearlMackenzie18}. He proposed that causal models would give the ability to ``anchor the elusive notions of science, knowledge, and data in a concrete and meaningful setting, and will enable us to see how the three work together to produce answers to difficult scientific questions,'' \cite{PearlMackenzie18}. We recount the important details of SCMs below.

\begin{definition}[Structural Causal Model \cite{pearl2009causality}]\label{def:cm}
	A structural causal model $\mathcal{M}$ is represented by a quadruple $\langle \mathbf{U}, \mathbf{V}, \mathbf{F},  P(\mathbf{U}) \rangle$ where
	\begin{enumerate}
		\item $\mathbf{U}$ is a set of exogenous (external) variables that are determined by factors outside the model.
		\item $\mathbf{V}$ is a set of endogenous (internal) variables that are determined by variables in $\mathbf{U}\cup\mathbf{V}$, i.e., $\mathbf{V}$'s values are determined by factors within the model.
		\item $\mathbf{F}$ is a set of structural equations from $\mathbf{U} \cup \mathbf{V} \to \mathbf{V}$.
		Specifically, for each $V\in \mathbf{V}$, there is a function $f_{V}\in \mathbf{F}$ mapping from $\mathbf{U} \cup (\mathbf{V}\backslash V) \to V$, i.e., $v=f_V ( \pa{V}, u_V )$, where $\pa{V}$ is a realization of a set of endogenous variables $\Pa{V} \in \mathbf{V} \setminus V$ that directly determines $V$, and $u_V$ is a realization of a set of exogenous variables that directly determines $V$. In other words, $f_{V}(\cdot)$ is a structural equation that expresses the value of each endogenous variable as a function of the values of the other variables in $\mathbf{U}$ and $\mathbf{V}$.
		\item $P(\mathbf{U})$ is a joint probability distribution defined over $\mathbf{U}$.
	\end{enumerate}
\end{definition}

In general, $f_{V}(\cdot)$ can be any type of equation. In some cases,  $f_{V}(\cdot)$ is assumed as a specific type, for example, the nonlinear additive function: $v = f_{V}(\pa{V})+\mathbf{u}_{V}$. If all exogenous variables in $\mathbf{U}$ are assumed to be mutually independent, then the causal model is called a \emph{Markovian model}; otherwise, it is called a \emph{semi-Markovian model}. 

The causal model $\mathcal{M}$ is associated with a causal graph $\CG = \langle \mathcal{V}, \mathcal{E} \rangle$ where $\mathcal{V}$ is a set of nodes (otherwise known as vertices) and $\mathcal{E}$ is a set of edges. Each node of $\mathcal{V}$ corresponds to a variable of $\mbf{V}$ in $\mathcal{M}$. Each edge in $\mathcal{E}$, denoted by a directed arrow $\rightarrow$, points from a node $X\in \mathbf{U} \cup \mathbf{V}$ to a different node $Y\in \mathbf{V}$ if $f_Y$ uses values of $X$ as input. A \emph{causal path} from $X$ to $Y$ is a directed path which traces arrows directed from $X$ to $Y$. The causal graph is usually simplified by removing all exogenous variables from the graph. In a Markovian model, exogenous variables can be directly removed without loss of information. In a semi-Markovian model, after removing exogenous variables, we also need to add dashed bi-directional edges between the children of correlated exogenous variables to indicate the existence of unobserved common cause factors, i.e., hidden confounders. Fig. \ref{fig:scm} depicts an example of a structural causal model.
\begin{figure}[t!]
    \centering
    \includegraphics[scale=.4]{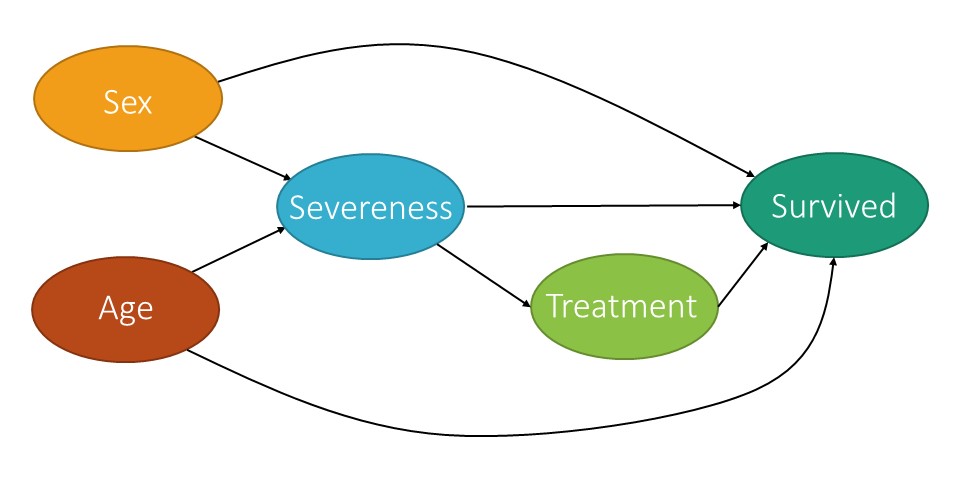}
    \caption{Example of a structural causal model depicting the relationships between variables that determine the survival of some disease.}
    \label{fig:scm}
\end{figure}

Quantitatively measuring causal effects in a causal model is facilitated with the $do$-operator \cite{pearl2009causality} which forces some variable $X$ to take on a certain value $x$, which can be formally denoted by $do(X = x)$ or $do(x)$. In a causal model $\mathcal{M}$, the intervention $do(x)$ is defined as the substitution of the structural equation $X=f_{X}(\Pa{X}, U_X)$ with $X=x$, which corresponds to a modified causal graph that has removed all edges into $X$ and in turn sets $X$ to $x$. For an observed variable $Y$ which is affected by the intervention, its interventional variant is denoted by $Y_{x}$.
The distribution of $Y_{x}$, also referred to as the post-intervention distribution of $Y$ under $do(x)$, is denoted by $P(Y_x=y)$ or simply $P(y_x)$.

Similarly, the intervention that sets the value of a set of variables $\mathbf{X}$ to $\mathbf{x}$ is denoted by $do(\mathbf{X} = \mathbf{x})$. The post-intervention distribution of all other attributes $\mathbf{Y}=\mathbf{V}\backslash \mathbf{X}$, i.e., $P(\mathbf{Y}=\mathbf{y}|do(\mathbf{X}=\mathbf{x}))$ or simply $P(\mathbf{y}|do(\mathbf{x}))$, can be computed by the truncated factorization formula \cite{pearl2009causality},
\begin{equation}\label{eq:do1}
P(\mathbf{y}|do(\mathbf{x})) = \prod_{Y\in \mathbf{Y}}P(y|PA(Y))\delta_{\mathbf{X}=\mathbf{x}},
\end{equation}
where $\delta_{\mathbf{X}=\mathbf{x}}$ means assigning attributes in $\mathbf{X}$ involved in the next term with the corresponding values in $\mathbf{x}$. Specifically, the post-intervention distribution of a single attribute $Y$ given an intervention on a single attribute $X$ is given by
\begin{equation}\label{eq:do}
P(y|do(x)) = \sum_{\mathbf{V}\backslash \{X,Y\},Y=y}\prod_{V\in \mathbf{V}\backslash \{X\}}P(v|PA(V))\delta_{X=x},
\end{equation}
where the summation is a marginalization that traverses all value combinations of $\mathbf{V}\backslash \{X,Y\}$.
We emphasize that $P(y|do(x))$ and $P(y|x)$ are not the same as the probabilistic distribution representing the statistical association ($P(y|x)$) is not equivalent to the interventional distribution ($P(y|do(x))$). We refer the interested readers to \cite{guo2020survey} for a discussion of this difference in relation to confounding bias, back-door criterion, and causal identification. 

By using the $do$-operator, the total causal effect is defined as follows.%
\begin{definition}[Total Causal Effect \cite{pearl2009causality}]\label{def:te}
	The total causal effect (TCE) of the value change of $X$ from $x_0$ to $x_1$ on $Y=y$ is given by
	\[ \TCE(x_1, x_0) = P(y_{x_1}) - P(y_{x_0}). \]
\end{definition}

The total causal effect is defined as the effect of $X$ on $Y$ where the intervention is transferred along all causal paths from $X$ to $Y$. Different from the total causal effect, the controlled direct effect (CDE) measures the effect of $X$ on $Y$ while holding all the other variables fixed. 

\begin{definition}[Controlled Direct Effect]\label{def:cde}
	The controlled direct effect (CDE) of the value change of $X$ from $x_0$ to $x_1$ on $Y=y$ is given by
	\[ \mathrm{CDE}(x_1, x_0) = P(y_{x_1,\mathbf{Z}}) - P(y_{x_0,\mathbf{Z}}) \] where $\mathbf{Z}$ is the set of all other variables. 
\end{definition}

In \cite{pearl2013direct}, Pearl introduced the causal mediation formula to decompose total causal effect into natural direct effect (NDE) and natural indirect effect (NIE).

\begin{definition}[Natural Direct Effect]\label{def:nde}
	The natural direct effect (NDE) of the value change of $X$ from $x_0$ to $x_1$ on $Y=y$ is given by:
	\[ \mathrm{NDE}(x_1, x_0) = P(y_{x_1, \mathbf{Z}_{x_0}}) - P(y_{x_0}) \] where $\mathbf{Z}$ is the set of mediator variables and $P(y_{x_1, \mathbf{Z}_{x_0}})$ is the probability of $Y=y$ had $X$ been $x_1$ and had $\mathbf{Z}$ been the value it would naturally take if $X=x_0$. In the causal graph, $X$ is set to $x_1$ in the direct path $X \rightarrow Y$ and is set to $x_0$ in all other indirect paths. 
\end{definition}

\begin{definition}[Natural Indirect Effect]\label{def:ide}
	The natural indirect effect (NIE) of the value change of $X$ from $x_0$ to $x_1$ on $Y=y$ is given by:
	\[ \mathrm{NIE}(x_1, x_0) = P(y_{x_0, \mathbf{Z}_{x_1}}) - P(y_{x_0}). \] 
\end{definition}

NDE measures the direct effect of $X$ on $Y$ while NIE measures the indirect effect of $X$ on $Y$. NDE differs from CDE since the mediators $\mathbf{Z}$ are set to $\mathbf{Z}_{x_0}$. In other words, the mediators are set to the level that they would have naturally attained under the reference condition $X=x_0$. 

The problem with NIE is that it does not separate the fair (explainable) and unfair (indirect discrimination) effects. Path-specific effect \cite{pearl2009causality} is an extension of total causal effect in the sense that the effect of the intervention is transmitted only along a subset of causal paths from $X$ to $Y$. Let $\pi$ denote a subset of causal paths. The $\pi$-specific effect considers a counterfactual situation where the effect of $X$ on $Y$ with the intervention is transmitted along $\pi$, while the effect of $X$ on $Y$ without the intervention is transmitted along paths not in $\pi$. 

\begin{definition}[Path-specific Effect \cite{avin2005identifiability}] \label{def:pse}
	Given a causal path set $\pi$, the $\pi$-specific effect ($\PE_{\pi}$) of the value change of $X$ from $x_0$ to $x_1$ on $Y=y$ through $\pi$ (with reference $x_0$) is given by
	\[ \PE_{\pi}(x_1, x_0) = P(y_{x_1 \vert \pi, x_0 \vert \bar{\pi}}) - P(y_{x_0}), \]
	where $P(Y_{ x_1 \vert \pi, x_0 \vert \bar{\pi} })$ represents the post-intervention distribution of $Y$ where the effect of intervention $do(x_1)$ is transmitted only along $\pi$ while the effect of reference intervention $do(x_0)$ is transmitted along the other paths.
\end{definition}

It can be directly obtained from the definitions above that the total causal effect and path-specific effect have the following connections:
\begin{enumerate}
    \item if $\pi$ contains all causal paths from $X$ to $Y$, then $\PE_{\pi}(x_{1},x_{0})=\TCE(x_{1},x_{0})$
    \item for any $\pi$, we have $\PE_{\pi}(x_{1},x_{0}) + (-\PE_{\bar{\pi}}(x_{0},x_{1})) = \TCE(x_{1},x_{0})$
\end{enumerate}

Definition~\ref{def:te} and \ref{def:pse} consider the average causal effect over the entire population without any prior observations. In contrast, the effect of treatment on the treated considers the sub-population of the treated group. 

\begin{definition}[Effect of Treatment on the Treated] \label{def:ett}
	The effect of treatment on the treated (ETT) of intervention $X=x_1$ on $Y=y$ (with baseline $x_0$) is given by
		\[ \mathrm{ETT}_{x_1, x_0} = P(y_{x_1|x_0}) - P(y|x_0), \]
	where $P(y_{x_1|x_0})$ 	represents the counterfactual quantity that read as ``the probability of $Y$ would be $y$ had $X$ been $x_1$, given that in the actual world, $X=x_0$.''
\end{definition}

Generally, if we have certain observations about a subset of attributes $\mathbf{O}=\mathbf{o}$ and use them as conditions when inferring the causal effect, then the causal inference problem becomes a \emph{counterfactual inference} problem. This means that the causal inference is performed on the sub-population specified by $\mathbf{O}=\mathbf{o}$ only. Symbolically, conditioning the distribution of $Y_{x}$ on factual observation $\mathbf{O}=\mathbf{o}$ is denoted by $P(y_{x}| \mathbf{o})$.
The counterfactual effect is defined as follows.

\begin{definition}[Counterfactual Effect \cite{shpitser2008complete}]\label{def:ce}
	Given a factual condition $\mathbf{O}=\mathbf{o}$, the counterfactual effect (CE) of the value change of $X$ from $x_0$ to $x_1$ on $Y=y$ is given by
	\[ \CE(x_1, x_0|\mathbf{o}) = P(y_{x_1} | \mathbf{o}) - P(y_{x_0} | \mathbf{o}). \]
\end{definition}

In \cite{wu2019pcfairness}, the authors present a general representation of causal effects, called path-specific counterfactual effect, which considers an intervention on $X$ transmitted along a subset of causal paths $\pi$ to $Y$, conditioning on observation $\mathbf{O}=\mathbf{o}$. 

\begin{definition}[Path-specific Counterfactual Effect]\label{def:psce}
	Given a factual condition $\mathbf{O}=\mathbf{o}$ and a causal path set $\pi$, the path-specific counterfactual effect (PCE) of the value change of $X$ from $x_0$ to $x_1$ on $Y=y$ through $\pi$ (with reference $x_0$) is given by
	\[ \PCE_{\pi}(x_1, x_0|\mathbf{o}) = P(y_{x_1 \vert \pi, x_0 \vert \bar{\pi}}|\mathbf{o}) - P(y_{x_0}|\mathbf{o}). \]
\end{definition}

In \cite{malinsky2019potential}, the conditional path-specific effect is different from our notion in that, for the former, the condition is on the post-intervention distribution, and for the latter, the condition is on the pre-intervention distribution. 

\subsubsection{Potential Outcome Framework}
The potential outcome framework \cite{imbens2015causal}, also known as Neyman-Rubin potential outcomes or the Rubin causal model, has been widely used in many research areas to perform causal inference. It refers to the outcomes one would see under each treatment option. Let $Y$ be the outcome variable, $T$ be the binary or multiple valued ordinal treatment variable, and $\mathbf{X}$ be the pre-treatment variables (covariates). Note that pre-treatment variables are the ones that are not affected by the treatment. On the other hand, the post-treatment variables, such as the intermediate outcome, are affected by the treatment. 

\begin{definition}[Potential Outcome]\label{def:po}
Given the treatment $T = t$ and outcome $Y=y$, the potential outcome of the individual $i$, $Y_i(t)$, represents the outcome that would have been observed if the individual $i$ had received treatment $t$.  
\end{definition}

In practice, only one potential outcome can be observed for each individual. The observed outcome is called the factual outcome and the remaining unobserved potential outcomes are the counterfactual outcomes. The potential outcome framework aims to estimate potential outcomes under different treatment options and then calculate the treatment effect. The treatment effect can be measured at the population, treated group, subgroup, and individual levels. Without the loss of generality, we discuss the treatment effect under binary treatment. 

\begin{definition}[Average Treatment Effect]\label{def:ate}
Given the treatment $T = t$ and outcome $Y=y$, the average treatment effect (ATE) is defined as 
\[\mathrm{ATE} = \mathbb{E}[Y(t')-Y(t)]\]
where $Y(t')$ and $Y(t)$ are the potential outcome and the observed control outcome of the whole population, respectively. 
\end{definition}

\begin{definition}[Average Treatment Effect on the Treated]\label{def:att}
Given the treatment $T = t$ and outcome $Y=y$, the average treatment effect on the treated group ($ATT$) is defined as \[\mathrm{ATT} = \mathbb{E}[Y(t')-Y(t)|T=t].\]
\end{definition}

The average treatment effect answers the question of how, on average, the outcome of interest $Y$ would change if everyone in the population of interest had been assigned to a particular treatment $t'$ relative to if they had received another treatment $t$. The average treatment effect on the treated on the other hand details the average outcome would change if everyone who received one particular treatment $t$ had instead received another treatment $t'$. 
 
\begin{definition}[Conditional Average Treatment Effect]\label{def:cate}
Given the treatment $T = t$ and outcome $Y=y$, the conditional average treatment effect ($CATE$) is defined as 
 \[\mathrm{CATE} = \mathbb{E}[Y(t')-Y(t)|\mathbf{W}={w}]\] 
where $\mathbf{W}$ is a subset of variables defining the subgroup. 
\end{definition}

\begin{definition}[Individual Treatment Effect]\label{def:ite}
Given the treatment $T = t$ and outcome $Y=y$, the individual treatment effect (ITE) is defined as 
\[ \mathrm{ITE} = \mathbb{E}[Y_i(t')-Y_i(t)]\] 
where $Y_i(t')$ and $Y_i(t)$ are the potential outcome and the observed control outcome of individual $i$, respectively.
\end{definition}

The potential outcome framework relies on three assumptions:
\begin{enumerate}
    \item Stable Unit Treatment Value Assumption (SUTVA): requires the potential outcome observation on one unit be unaffected by the particular assignment of treatments to other units.
    \item Consistency Assumption: requires that the value of the potential outcomes would not change no matter how the treatment is observed or assigned through an intervention.
    \item Strong Ignorability (unconfoundedness) Assumption: is equal to the assumption that there are no unobserved confounders\footnote{A confounder is a pre-treatment variable that affects both treatment and outcome variables.}. 
\end{enumerate}

Under these assumptions, causal inference methods can be applied to estimate the potential outcome and treatment effect given the information of treatment variable and pre-treatment variables. We refer the interested readers to the survey \cite{yao2021survey} for various causal inference methods, including re-weighting, stratification, matching based, and representation based methods.  

\section{Philosophical Perspectives}
This section is devoted to explaining the philosophical underpinnings of both statistical and causal-based fair machine learning metrics. Many of the statistical-based fairness metrics in the fair machine learning literature correspond to the notions of distributive justice from social science literature \cite{gajane_formalizing_2018}. Here, we introduce the philosophical ideal of equality of opportunity (EOP) and its three main frames: formal EOP, substantive EOP, and luck-egalitarian EOP. We note that other definitions and ideals of egalitarianism exist, and are relevant to the discussion on fair machine learning \cite{arneson_equality_1989, cohen_currency_1989, dworkin_what_1981}, but we limit ourselves to the discussion below as it directly aligns with our classification of popular fairness metrics in Section 5. We direct interested readers to these mentioned works, along with corresponding surveys \cite{arneson_equality_2015, roemer_equality_2013}, for additional reading. Additionally, we introduce the philosophical notions of causality as proposed by historical philosophers Aristotle, Hume, and Kant as well as through the modern day descriptions by American philosopher David Lewis.

\subsection{Equality of Opportunity (EOP)}
In \cite{khan_fairness_2021}, they propose grounding current (and future) proposed fair machine learning metrics in the moral framework of equality of opportunity (EOP) \cite{sep-equal-opportunity}. EOP is a political ideal that is opposed to assigned-at-birth (caste) hierarchy, but not to hierarchy itself. In a caste hierarchy, a child normally acquires the social status of their parents. Social mobility may be possible, but the process to rise through the hierarchy is open to only specific individuals depending on their initial social status. In contrast to a caste hierarchy, EOP demands that the social hierarchy is determined by a form of equal competition among all members of the society. 

From a philosophical perspective, EOP is a principle that dictates how desirable positions, or opportunities, should be distributed among members of a society. As a moral framework, EOP allows machine learning practitioners to see fairness notions' motivations, strengths, and shortcomings in an organized and comparative fashion. Additionally, it presents moral questions that machine learning practitioners must answer to guide the construction of a system that has fairness ideals that satisfy their desired values \cite{khan_fairness_2021}. Further, it allows practitioners to understand and appreciate why there may be disagreement when it comes to choosing a specific fairness metric, as different people will have different moral beliefs about what fairness and equality mean. Different conceptions of EOP, such as formal, substantive, and luck-egalitarian EOP, interpret the idea of competing on equal terms in different ways.

\begin{figure}[h!]
    \centering
    \includegraphics[scale=.25]{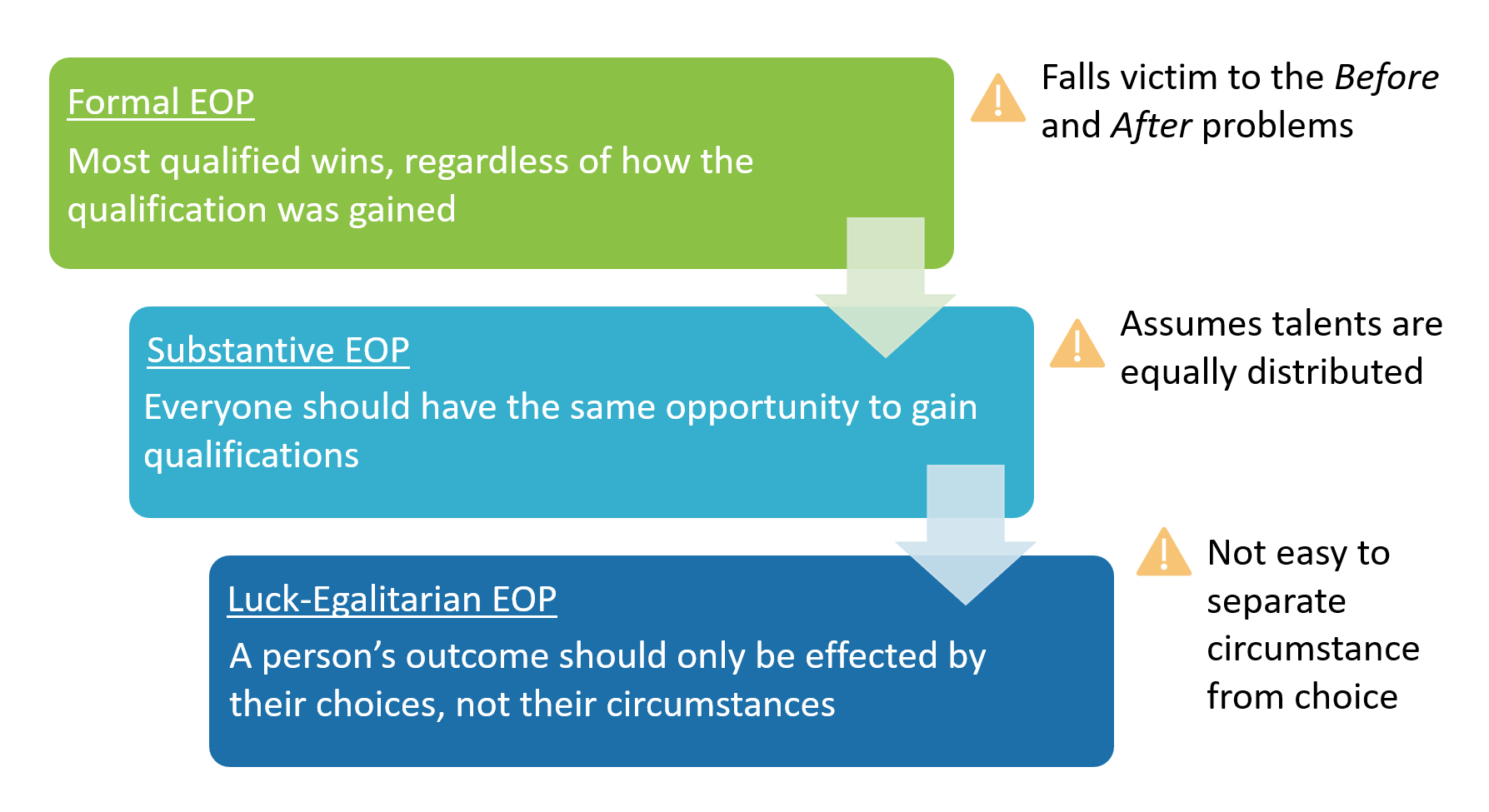}
    \caption{Definitions and warnings for popular equality of opportunity modes.}
    \label{fig:philo_def}
\end{figure}

\subsubsection{Formal EOP}
Formal EOP emphasizes that any desirable position (in a society, or more concretely, a job opening) is available and open to everyone. The distribution of these desirable positions follows according to the individual's relevant qualifications, and in this setting, the most qualified always wins. Formal EOP takes a step in the direction of making decisions based on relevant criteria rather than making them in a blatant discriminatory fashion \cite{khan_fairness_2021}. In fair machine learning settings, formal EOP has often been implemented as fairness through blindness or fairness through unawareness \cite{kusner_loftus_russell_silva}. In other words, formal EOP-based metrics strip away any irrelevant marginalization attributes, such as race or gender, before training is performed. 

However, while formal EOP has the benefit of awarding positions based on actual qualifications of an individual, and in excluding irrelevant marginalization information, it makes no attempt to correct for \textit{arbitrary} privileges. This includes unequal access to opportunities that can lead to disparities between individuals' qualifications. Examples of this situation can be seen in the task of predicting prospective students' academic performance for use in college admission decisions \cite{foulds_intersectional_2019}. Individuals belonging to marginalized or non-conforming groups, such as Black or LGBTQIA+\footnote{Lesbian, Gay, Bisexual, Transgender, Queer, Intersex, and Asexual} students, are disproportionately impacted by the challenges of poverty, racism, bullying, and discrimination. An accurate predictor for a student's success therefore may not correspond to a fair decision-making procedure as the impact of these challenges create a tax on the ``cognitive-bandwidth'' of non-majority students, which in turn, affects their academic performance \cite{berk2017fairness,verschelden_bandwidth_2017}.

The problem of not accounting for arbitrary privileges can be broken down into two main problems: the Before and After problems. In the Before problem, arbitrary and morally irrelevant privileges weigh heavily on the outcomes of formally fair competitions as people with more privilege are often in a better position to build relevant qualifications. This can be seen in the problem described above of predicting students' performance for admissions. The After problem is an effect of formally fair, but not arbitrary-privilege aware, competition in that the winner of a competition (e.g., getting hired or admitted to a top-tier college) is then in the position to gain even more success and qualifications. It introduces a compounding snow-ball effect in that winners win faster, but losers also lose faster \cite{khan_fairness_2021}. Overall, formal EOP compounds both privilege and disprivilege, a term referred to as \textit{discrimination laundering} \cite{green_discrimination_2017}.

\subsubsection{Substantive (Rawls) EOP} Substantive EOP goes one step beyond formal EOP, and address the discrimination laundering problem in that it requires all individuals have the same opportunity to gain qualifications. It aims to give everyone a fair chance at success in a competition. For example, making all extra curricular activities and high-school opportunities equally available to all students regardless of wealth or social status. Substantive EOP is often equated with Rawls' fair EOP which states that all individuals, regardless of how rich or poor they are born, should have the same opportunities to develop their talents in order to allow people with the same talents and motivation to have the same opportunities \cite{rawls_theory_1999}. In fair machine learning, substantive EOP is often implemented through metrics such as statistical parity and equalized odds, which assume that talent and motivation are equally distributed among sub-populations. 

But, the assumption that these talents are equally distributed, often does not hold in practice. By the time a machine learning system is being used to make a decision, it is normally too late to provide individuals with the opportunity to develop qualifications. In this lens, fair machine learning has re-written Rawls' EOP to say that a competition must only measure a candidate on the basis of their talent, while ignoring qualifications that reflect the candidates unequal developmental opportunities prior to the point of the competition \cite{khan_fairness_2021}.

\subsubsection{Luck-Egalitarian EOP} Furthering the ideals of substantive EOP, luck-egalitarian EOP enforces that a person's outcome should be affected only by their choices, not their circumstances \cite{segall_equality_2013}. For example, a student with rich parents who did not try hard in their studies should not have an advantage over a student with poor parents who did work hard in being admitted to a university. Overall, luck-egalitarian EOP is an attractive idea, but the difficulty of separating choice from circumstance is non-trivial, and in practice, quite difficult to implement.

A few solutions for separation have been proposed. Economist and political scientist John Roemer proposed instead of trying to separate an individuals qualifications into effects of consequence and choice, we should instead control for certain matters of consequence (such as, race, gender, and disability) that will impact a person's access to opportunities to develop qualifications \cite{roemer_equality_2013}. While this solution solves the separation problem, another issue of sub-group comparison emerges. We can compare apples to apples, and oranges to oranges, but we are now unable to compare apples to oranges \cite{khan_fairness_2021}. Unfortunately, the EOP frameworks offer no solution to this problem of overall ranking. 

Another problem in luck-egalitarian EOP is the separation of efforts in addition to circumstance \cite{khan_fairness_2021}. It may be the case that a wealthy student works hard at their studies, i.e., the circumstance of being wealthy interacts with the effort of the student. This effect of entanglement is nearly impossible to separate. But, fortunately, this separation is only required when the circumstance gives access to a broad range of advantages. For instance, if a student's family wealth status allows them to gain an advantage over all other students in almost every competition (not just university admission, but also job hiring or access to other opportunities), then there is a fairness problem. This is because there is an indication that the arbitrary privilege or circumstance, and not the relevant skill, is being used as the basis for decision. On the other hand, if the student only has the advantage in the admissions process, then it could be due to their effort rather than their circumstance, and we may or may not have a matter of unfairness where we need to separate effort from circumstance.

\subsection{Causality and Counterfactuals}
\begin{figure}
\centering
    \includegraphics[scale=.4]{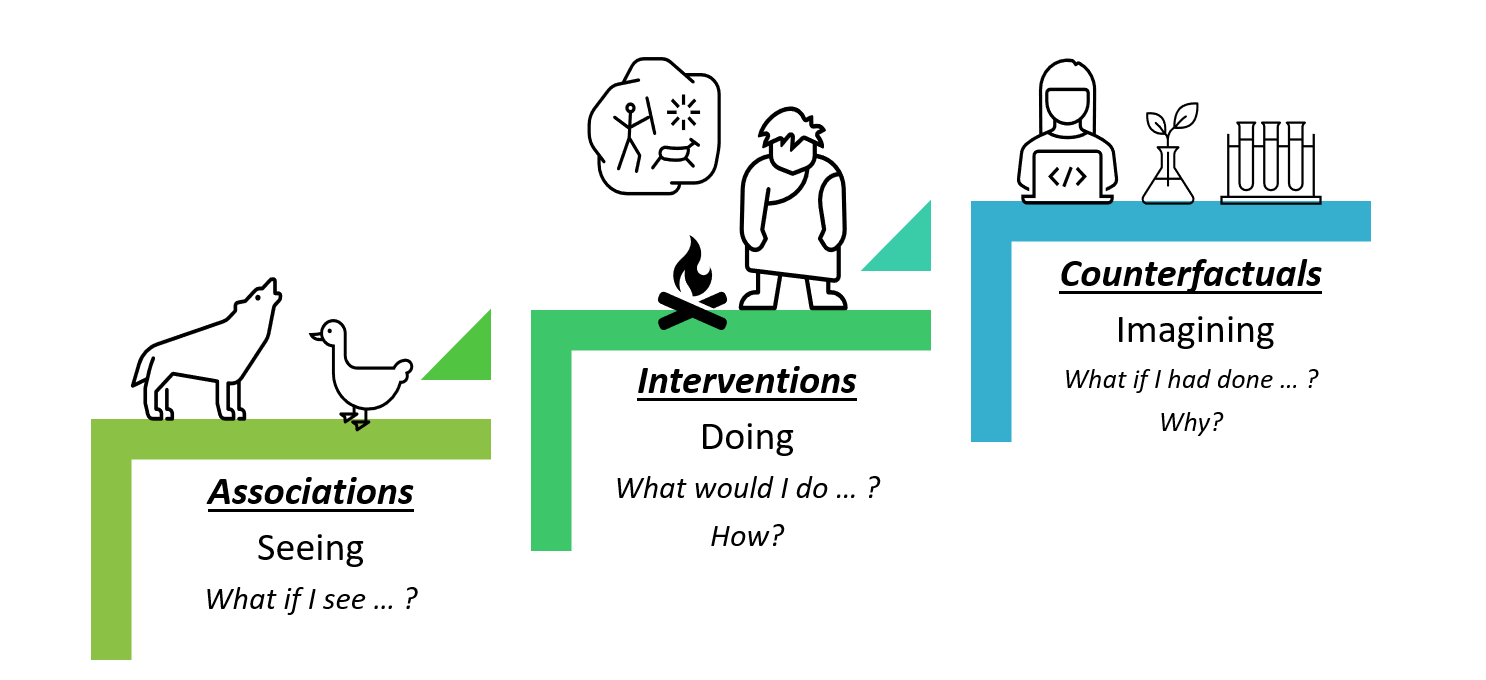}
    \caption{Pearl's Ladder of Causation. The first rung, associations, only allows predictions based on passive observations. The second rung, interventions, not only relies on seeing, but also changing what is. Rung three, counterfactuals, deals with the imaginary, or what might have been. See Section 6 for a more in-depth view of the Ladder of Causation.}
\label{fig:ladderofcause}
\end{figure}

``\textit{Felix, qui potuit rerum cognoscere causas}," Latin poet Virgil wrote in 29 BC, ``fortunate is he, who is able to know the causes of things." Several hundreds of years later in 1953, Albert Einstein reaffirmed this statement, proclaiming that ``the development of Western science is based on two great achievements: the invention of the formal logical system (in Euclidean geometry) by the Greek philosophers, and the discovery of the possibility to find out causal relationships by systematic experiment (during the Renaissance).'' But even thousands of years before Virgil and Einstein, dating back to 3,000-2,000 BC, the Sumerian civilization often wrote of causal relationships (termed as omens in their day) that had the form of ``if $x$ happens, then $y$ will happen'' \cite{Frosini2006}. Causal and counterfactual reasoning is ``part of the mental apparatus that makes us human", so it is not surprising that we can find causal and counterfactual statements ``as far back as we want to go in human history \cite{PearlMackenzie18}."


The first formal investigation into causality was done by the Greek philosopher Aristotle. In 350 BC, Aristotle published his two famous treatise, \textit{Physics} and \textit{Metaphysics}. In these treatise, Aristotle not only opposed the previously proposed notions of causality for not being grounded in any solid theory \cite{sep-aristotle-causality}, but he also constructed a taxonomy of causation which he termed ``the four causes." In order to have proper knowledge, he deemed that we must have grasped its cause, and that giving a relevant cause is necessary and sufficient in offering a scientific explanation. His four causes can be seen as the four types of answers possible when asked a question of ``why."
\begin{enumerate}
    \item The material cause: ``that out of which'' (something is made). E.g, the marble of a statue.
    \item The formal cause: ``the form'', ``the account of what-it-is-to-be.'' E.g, the shape of the statue. 
    \item The efficient cause: ``the primary source of the change or rest.'' E.g, the artist/sculptor of the marble statue.
    \item The final cause: ``the end, that for the sake of which a thing is done.'' E.g, the creation of a work of art.  
\end{enumerate}
Despite giving four causes, Aristotle was not committed to the idea that everything explanation had to have all four. Rather, he reasoned that any scientific explanation required \textit{up to} four kinds of cause \cite{sep-aristotle-causality}. It is important to note, however, that Aristotle made no mention of the theory of counterfactuals.

To find a philosopher who centered causation around counterfactual reasoning, we look to 18th century Scottish philosopher David Hume. He rejected Aristotle's taxonomy and instead insisted on a single definition of cause. This is despite the fact that he himself could not choose between two different, and later found to be incompatible, definitions \cite{PearlMackenzie18}. In his \textit{Treatise of Human Nature}, Hume states that ``several occasions of everyday life, as well as the observations carried out for scientific purposes, in which we speak of a condition A as a cause and a condition B as its effect, bear no justification on the facts, but are simply based on our habit of observing B after having observed A" \cite{Frosini2006}. In other words, Hume believed that the cause-effect relationship was a sole product of our memory and experience \cite{PearlMackenzie18}. Later, in 1739, Hume published \textit{An Enquiry Concerning Human Understanding} in which he framed causation as a type of correlation: ``we may define a cause to be an object followed by another, and where all the objects, similar to the first, are followed by objects similar to the second. Or in other words, where, if the first object had not been, the second never had existed." While he tried to pass these two definitions off as one by using ``in other words," David Lewis pointed out that the second statement is contradictory to the first as it explicitly invokes the notion of a counterfactual which, cannot be observed, only imagined \cite{PearlMackenzie18}.

The incorporation of counterfactuals into causal theory went largely ignored by philosophers through the 19th and 20th centuries who instead tried to justify Hume's first definition through the theory of probabilistic causation\footnote{See \textit{The Book of Why: The New Science of Cause and Effect} by Judea Pearl and Dana Mackenzie for the entangled history of causation, counterfactuals, and statistics.} \cite{PearlMackenzie18}. It was not until the 1973 publication of \textit{Counterfactuals} by David Lewis that this trend would change. He called for abandoning the first definition all together, and to instead interpret ``A has caused B" as ``B would not have occurred if not for A" as it more closely aligned with human intuition \cite{lewis-1973}. He reasoned that we as humans make counterfactual judgements easily, often, and that we assign these counterfactual situations truth values and probabilities with the same confidence that we have for factual statements by envisioning ``possible worlds" in which the counterfactual statements are actually true \cite{PearlMackenzie18}. The idea of other possible worlds was met with mass criticism, and it would take the work of Judea Pearl to defend Lewis's claims. 

Much of modern day philosophy and theory on causality can be attributed to Judea Pearl. Pearl, an Israeli-American computer scientist and philosopher, is credited with the development of causal and counterfactual inference theory based on structural models as well as the Ladder of Causation, which is shown in Fig. \ref{fig:ladderofcause}. He believed that Lewis's critics (as well as Lewis himself) were thinking the wrong way about the ``existence of other worlds." He argued that we do not need to debate about if these world exist or not, but rather we are simply hypothesizing that people are capable of imagining other worlds in their heads, which relates to the ``amazing uniformity of the architecture of the human mind" \cite{PearlMackenzie18}. He believed that we humans experience the same world, which in turn allows us to experience the same mental model of its causal structure and binds us together into communities \cite{PearlMackenzie18} (see Section 7.4 for critique of this idea). Additionally, he believed that while most philosophers thought the structural causal model to be just one possible implementation of Lewis's logic, SCM were actually extremely similar to the way that our brains represent counterfactuals \cite{PearlMackenzie18}.

Beyond Lewis and Pearl, several other modern day philosophers have proposed notions of causality with vary degrees of ``realness''. For instance, Bertrand Russell and Karl Pearson denied any scientific status to the cause-effect relationship \cite{Frosini2006}. Other statisticians and philosophers of science opposed this line of thinking. One example is Wesley C. Salmon, who wrote: ``to many scientists uncorrupted by philosophical training, it is evident that causality plays a central role in scientific explanation" \cite{humphreys_salmon_2000}.

\section{Legal Perspectives}
The aspects of legality-based fairness will change depending on what country hosts the party who is wanting to create a fair machine learning algorithm. What is considered legally fair in one country may contradict or oppose that of another. For this reason, we specify that this article is framed from the viewpoint of Western law, specifically that of U.S. anti-discrimination and criminal law. We chose this frame of reference since most fair machine learning literature focuses on terminology and legal aspects proposed in this body of law \cite{xiang2019legalcompatibility}. 

\begin{figure}[h!]
    \centering
    \includegraphics[scale=.3]{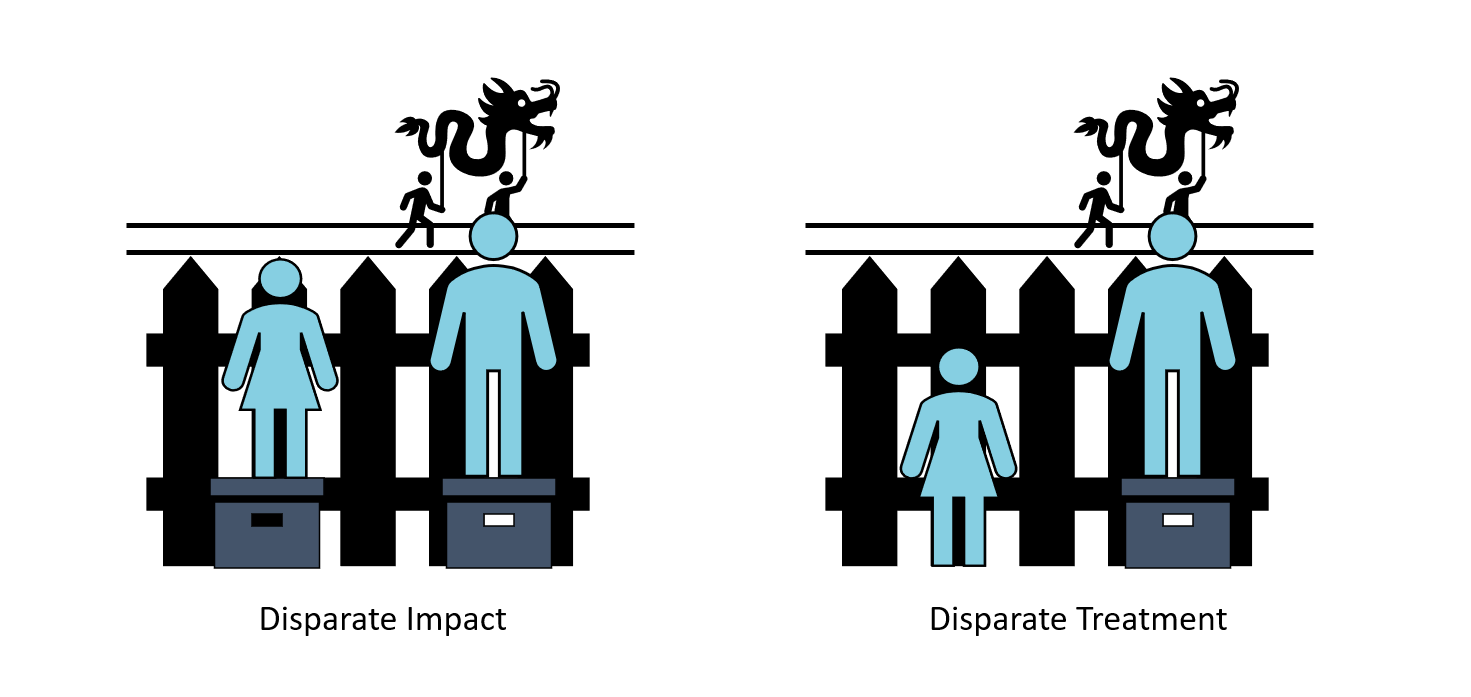}
    \caption{Graphical depiction of disparate impact vs. disparate treatment. For the disparate impact (left) side, while both people were given boxes to stand on in order to see over the fence, the person on the left is disproportionately impacted by being shorter. For disparate treatment, the person on the left if directly being discriminated against by not being given a box to stand on.}
    \label{fig:dis_imp_dis_treat}
\end{figure}


\subsection{Disparate Impact}
Disparate impact occurs when members of a marginalized class are negatively affected more than others when using a formally neutral policy or rule \cite{pessach_algorithmic_2020}. In other words, it is unintentional or indirect discrimination. This statutory standard was first formalized in the 1971 landmark U.S. Supreme Court case of \textit{Griggs vs. Duke Power Co.}, where Duke Power Co. required employees to have a high school diploma in order to be considered for a promotion \cite{griggs1971}. The Supreme Court noted that Title VII of the 1964 Civil Rights Act was violated in this situation as Duke Power Co.'s diploma requirement, which had little relation to the overall job performance, prevented a disproportionate number of Black employees from being promoted, and thus, deemed it to have an illegal (unjustified) disparate impact. In fair machine learning literature, disparate impact is often framed as occurring when the outcomes of classification differ across different subgroups, even if this outcome disparity is not intentional \cite{lipton2019does}.

However, disparate impact in of itself is not illegal. In U.S. anti-discrimination law, indirect discrimination in employment is not unlawful if it can be justified by a ``legitimate aim" or a genuine occupational requirement and/or business necessity \cite{weerts_introduction_2021}. For example, in the \textit{Griggs vs. Duke Power Co.} case, if the high school diploma requirement was shown to be necessary for overall job success, then the resulting disparities would have been legal \cite{corbettdavies2018measure}.

The most common measure of disparate impact is the four-fifths rule proposed in the Uniform Guidelines on Employee Selection Procedures in the Code of Federal Regulation \cite{four-fifths}. This rule states that if the selection rate for a certain group is less than 80\% of the group with the highest selection rate, then there is a disparate impact (or adverse impact) on that group. For a concrete example, in the \textit{Griggs vs. Duke Power Co.} case, the selection rate of Black individuals for promotion was 6\% while the selection rate of White individuals was 58\%. By taking the ratio of the two selection rates we get $\frac{6}{58} = 10.3\%$ which is well below the 80\% threshold. In Section 5.1 we discuss a fairness metric that aligns well with the disparate impact four-fifths rule. 

\subsection{Disparate Treatment}
In direct contrast with disparate impact, disparate treatment occurs when an individual is intentionally treated different based on their membership of a marginalized class. The Uniform Guidelines on Employee Selection Procedures states that: ``disparate treatment occurs where members of a race, gender, or ethnic group have been denied the same employment, promotion, membership, or other employment opportunities as have been available to other employees or applicants'' \cite{UGESP-disparate-treatment}. The key legal question behind disparate treatment is whether the alleged discriminator's actions were motivated by discriminatory intent \cite{ four-fifths}. In fair machine learning literature, a metric is often described as adhering to disparate treatment law if it does not use any marginalization attributes in the decision making process \cite{jagielski2019differentially}.

\subsection{Anti-Classification and Anti-Subordination}
Two main principles that were produced by the U.S. anti-discrimination legislation are anti-classification and anti-subordination \cite{xiang2019legalcompatibility, balkin_siegel_2003}. The anti-classification principle notes that the U.S. government may not classify people (either overtly or surreptitiously) by use of a marginalization attribute such as race or gender. Anti-subordination, on the other hand, is the idea that anti-discrimination laws should serve a grander purpose and should actively be aiming to tackle societal hierarchies between different groups. In other words, it argues that it is inappropriate for any group to be ``subordinated'' in society, and that empowering marginalized groups, even at the sake of the non-marginalized group, should be the main focus of anti-discrimination law \cite{colker_2021}. Anti-classification is often seen in fair machine learning as fairness through unawareness \cite{dwork_fairness_2011}, although, many note that the exclusion of the marginalization attributes can lead to discriminatory solutions \cite{corbettdavies2018measure}. Anti-subordination is less popular in the fair machine learning literature, but some work using debiasing techniques could potentially have the effect of anti-subordination \cite{xiang2019legalcompatibility, lipton2019does}.

\subsection{Critique of Current Fair ML Metric Alignment}
Several works exist that critique the alignment of current fairness metrics with disparate impact and disparate treatment. Xiang and Raji note that both disparate impact and disparate treatment were developed with human discriminators in mind, and simply replacing human decision makers with algorithmic ones is often not appropriate \cite{xiang2019legalcompatibility}. They state that ``intent is an inherently human characteristic'', and the common fair machine learning characterization of disparate treatment as not using marginalization class variables in an algorithm should be contested. They also note that simply accounting for disproportionate outcomes is not enough to prove disparate impact. It is only the first step of a disparate impact case, and there is only liability if the defendant cannot justify the outcomes by using non-discriminatory rationals. 

Additionally, \cite{grgic-hlaca_beyond_nodate} notes that while most work on fair machine learning has focused on achieving a fair distribution of decision outcomes, little to no attention has been paid to the overall decision process used to generate the outcome. They note that this is at the determent of not incorporating human moral sense for whether or not it is fair to use a feature in a decision making scenario. To this end, they propose notions of procedural fairness that utilize several considerations that are overlooked in distributive fairness cases, such as feature volitionality, feature reliability, feature privacy, and feature relevance.

\begin{table}[t!]\small
\centering
\setlength{\extrarowheight}{0pt}
\addtolength{\extrarowheight}{\aboverulesep}
\addtolength{\extrarowheight}{\belowrulesep}
\setlength{\aboverulesep}{0pt}
\setlength{\belowrulesep}{0pt}
\caption{Definitions of legal causation terminology.}
\begin{tabular}{cl} 
\toprule
\rowcolor[rgb]{0.592,0.792,0.792} \textbf{\textcolor[rgb]{0,0.502,0.502}{Term}} & \multicolumn{1}{c}{\textbf{\textcolor[rgb]{0,0.502,0.502}{Definition}}}                                                                                                                                                                                                \\ 
\hline
Causation                                                                       & The causing or producing of an effect.                                                                                                                                                                                                                                 \\ 
\hline
\begin{tabular}[c]{@{}c@{}}Factual\\(“but for”)\\Causation\end{tabular}         & \begin{tabular}[c]{@{}l@{}}An act or circumstance that causes an event, where the event \\would not have happened had the act or circumstance not occurred.\end{tabular}                                                                                               \\ 
\hline
\begin{tabular}[c]{@{}c@{}}Proximate\\Causation\end{tabular}                    & A cause that is legally sufficient to result in liability.                                                                                                                                                                                                             \\ 
\hline
\begin{tabular}[c]{@{}c@{}}Intervening\\Factor\end{tabular}    & \begin{tabular}[c]{@{}l@{}}An event that comes between the initial event (in a sequence of\\events) and the end result, thereby altering the natural course of\\events that might have connected a wrongful act to an injury.\end{tabular}                             \\ 
\hline
\begin{tabular}[c]{@{}c@{}}Superceding\\Factor\end{tabular}                     & \begin{tabular}[c]{@{}l@{}}An intervening act that the law considers sufficient to override the\\cause for which the original actor is responsible, thereby relieving\\the original actor of liability for the result.\end{tabular}  \\
\bottomrule
\end{tabular}
\end{table}

\subsection{Procedural Fairness}
Machine learning fairness researchers are beginning to recognize that most of the proposed fairness metrics have focused solely on distributive fairness (i.e., disparate impact or disparate treatment) \cite{morse_teodorescu_awwad_kane_2021, 10.1145/3287560.3287598, saxena2019fairness}. But, over the last few years, commentary on utilizing procedural fairness as a grounding for fair machine learning metrics have been published \cite{morse_teodorescu_awwad_kane_2021, grgic-hlaca_beyond_nodate}. Procedural fairness refers to the fairness of the decision making \textit{process} that lead to the outcome. This contrasts with distributive fairness which refers to making sure that the \textit{outcomes} of a process (or classification task) are fair \cite{grgic-hlaca_beyond_nodate}.

However, there has been some push-back on this view of procedural fair machine learning, though. Xiang and Raji note that the term ``procedural fairness'' as described in the fair machine learning literature is a narrow and misguided view of what procedural fairness means from a legal lens \cite{xiang2019legalcompatibility}. Procedural justice aims to arrive at a just outcome through an iterative process as well as through a close examination of the set of governing laws in place that guide the decision-maker to a specific decision \cite{xiang2019legalcompatibility,citron_scored_2014}. They pose that the overall goal of procedural fairness in machine learning should be re-aligned with the aim of procedural justice by instead analyzing the system surrounding the algorithm, as well as its use, rather than simply looking at it from the specifics of the algorithm itself.

\subsection{Causality and the Law}
When proving guilt in a criminal court case, the prosecution is required to prove that the defendant's action was the legal cause of the result. Establishing this causal relationship is a two-step process in which the prosecution first establishes \textit{factual} (``but for'') causation and then determines if there is \textit{proximate} causation \cite{Kaplan2012}. Different to criminal law, anti-discrimination law is motivated by ``because of'' causality, i.e., decision cannot be made ``because of'' an individual's marginalization attribute \cite{xiang2019legalcompatibility}. In this section, we will focus on criminal law as it is the most intuitive implementation of causality in the legal field.

To prove factual causation, the prosecutor must depict that ``but for'' the defendant's specific actions, the result or injury would not have happened as it did, when it did, or at all. However, this does not mean that the prosecution has to prove that the defendant's actions were the sole cause of the result, as their actions may have been combined with those of another person, or another circumstance, that all contributed to the final result. An exception to factual causation is when the chain of events caused by the defendant's actions is effectively broken. These intervening factors must be unforseeable. For instance, if the defendant's actions put the victim in the hospital (in a non-critical condition), but by the effect of gross medical malpractice (just negligent medical care is foreseeable), they die. Then, the defendant would most likely be charged for assault, but not homicide. 

After proving factual causation, the prosecution must then prove proximate causation. Typically, proximate cause issues arise when the final result occurs through an unexpected manner. For instance, if the defendant shot the victim in the arm, who then while running away from the defendant, tripped, fell, and cracked their skull, which resulted in their death a few moments later, then the defendant's actions were the proximate cause of the victims death. The general rule is that the defendant’s actions will be regarded as the proximate cause of a result if the result occurred as a “natural and probable consequence” of the acts, and there was no intervening factor sufficient enough to break the chain of causation \cite{geiger1968, causation_2021}.

\begin{figure}[t!]
    \centering
    \includegraphics[scale=.3]{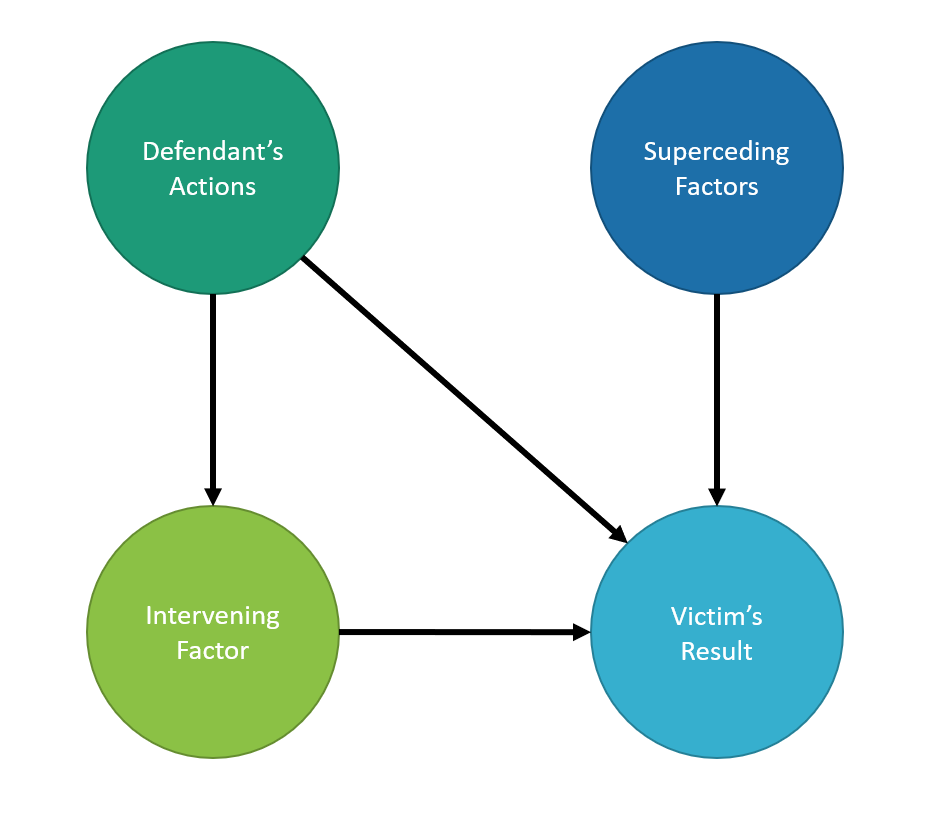}
    \caption{Graphical model displaying the causal relationships between the defendant's actions, the victim's result, along with intervening and superceding Factors.}
    \label{fig:causal}
\end{figure}
 Using causal graphical modeling, we can display the relationship of causation in the law as shown in Fig. \ref{fig:causal}. When relating to casual-based fairness metrics (see Section 6), the legal notion of causality closely aligns with the idea of path-specific causal effect. In this case, instead of computing the direct and indirect effects, path-specific causal effect isolates the contribution of the effect along a specific group of paths \cite{chiappa2019pathspecific}. This is similar to how lawyers and judges make decisions on if a certain action caused a certain effect. For instance, they reason if the intervening factor (if there is one) played a role in the victim's result and if this intervening factor ``broke the chain'' of the defendants actions in a way that no longer holds them liable. This would result in turning \textit{Defendant's Actions} $\to$ \textit{Intervening Factor} $\to$ \textit{Victim's Result} to simply be \textit{Intervening Factor} $\to$ \textit{Victim's Result} as shown in Figure \ref{fig:causal}.

While judges and policy makers are apt at determining cause in the real-world, they often do not understand what indicates causality in the use of algorithms. Xiang and Raji give the following example \cite{xiang2019legalcompatibility}:
\begin{displayquote}
``The U.S. Department of Housing and Urban Development recently proposed a rule that would create a safe harbor from disparate impact claims for housing algorithms that do not use [marginalization] class variables or close proxies. The rule seeks to implement the Supreme Court decision in the Texas Department of Housing and Community Affairs v. Inclusive Communities Project, Inc. case, which required housing discrimination plaintiffs to `show a causal connection between the Department's policy and a disparate impact.' That said, from a statistical perspective, the presence or absence of [marginalization] class variables (or close proxies) in the algorithm does not indicate the presence or absence of a causal connection.''
\end{displayquote}
With this confusion, it is important that the fair machine learning community makes efforts to communicate with the legal field about what causality means from a technical perspective, and help them to translate what the corresponding nature of causality would then be in legal terms.

\section{Statistics-Based Fairness Metrics}
\label{sec:sfairness}

In this section, we organize several popular statistics-based fair machine learning metrics into categories based on what attributes of the system they use (e.g., the predicted outcome, the predicted and actual outcomes, or the predicted probability and actual outcome), which formal statistical criterion (independence, separation, or sufficiency) it aligns with as proposed in \cite{barocas-hardt-narayanan}, as well as which philosophical ideal serves as its foundation (e.g., Rawls' (substantive) EOP or Luck-Egalitarian EOP) by using the classification procedure explained in \cite{heidari_moral_2019, lefranc_equality_2009}. Fig. \ref{fig:conf-matrix} shows our classification through an extended confusion matrix and Table 4 summarizes our main conclusions. Additionally, at the end of this section, we devote space to discussing individual fairness and the (apparent) differences between individual and group fairness metrics.

\begin{figure}[h!]
    \centering
    \includegraphics[scale=.35]{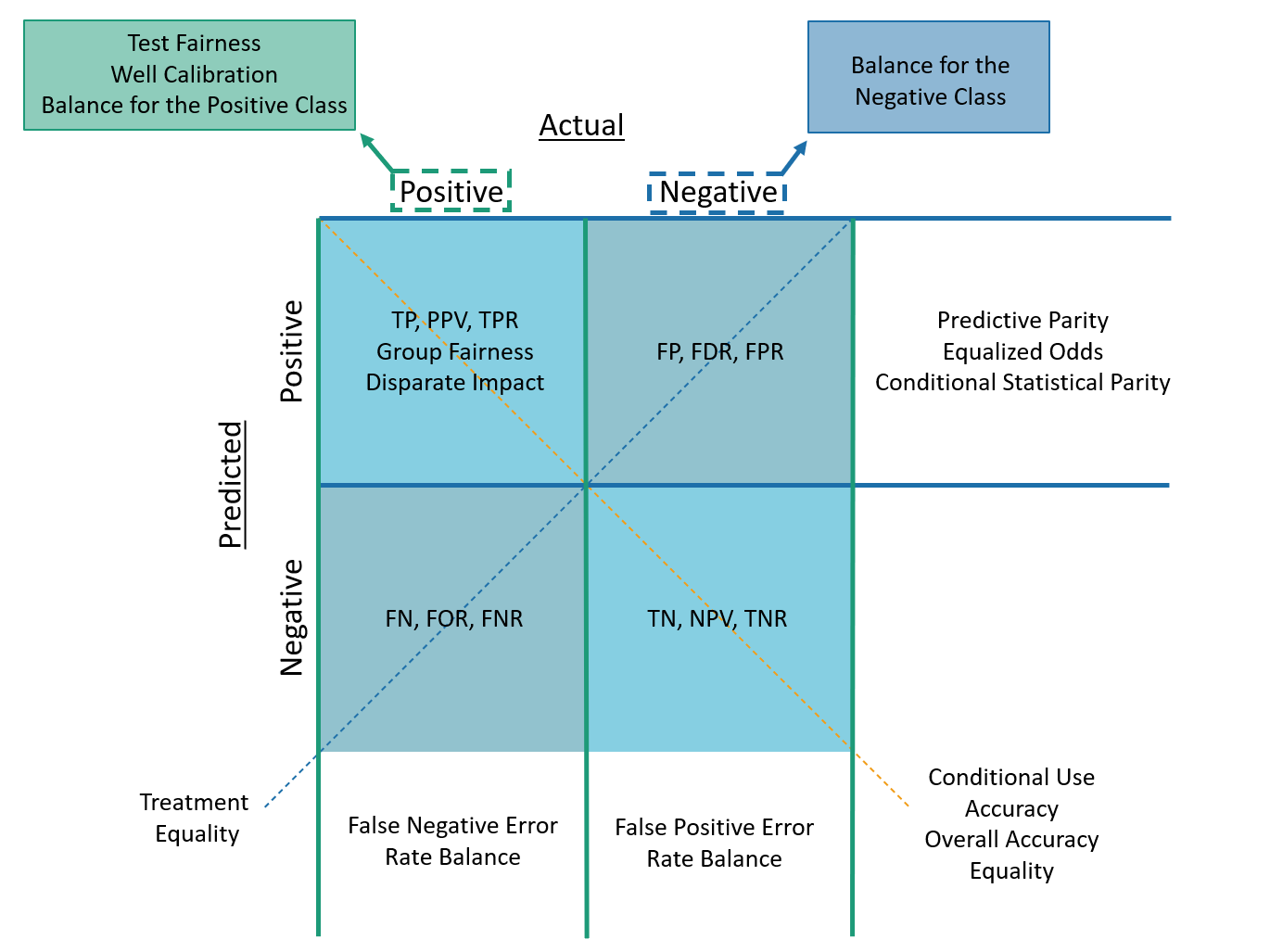}
    \caption{Classification of each fairness metric into our extended confusion matrix. Metrics were classified by which statistics they used. For instance, since equalized odds uses positive predictive rates, it is on the row where the predicted value is positive. Similarly, since test fairness says that all groups should have equal probability of belonging to the positive class, it is classified in the top left square related to the actual label of positive. }
    \label{fig:conf-matrix}
\end{figure}
For clarity, we note that this section uses the following variables: $S=1$ is the marginalized or minority group, $S = 0$ is the non-marginalized or majority group, $\hat{Y}$ is the predicted value or class (i.e., label), and $Y$ is the true or actual label. Additionally, to provide a background for our future explanations, we state a relaxed, binary, version of two important definitions for Rawls' EOP and Luck-Egalitarian EOP for supervised learning proposed by Heidari et al \cite{heidari_moral_2019}:

\begin{definition}[Rawls' EOP for Supervised Learning]
 Predictive model $h$ satisfies Rawlsian EOP if for all $s\in S=\{0,1\}$ and all $y\in Y = \{0,1\} $:
 $$F^h(U \leq u\;|\;S = 0 \cap Y = y) = F^h(U \leq u \;|\; S = 1 \cap Y = y)$$
 where $F^h(U \leq u)$ specifies the distribution of utility across individuals under the predictive model $h$. I.e., difference between the individuals' actual effort, $A$, and their circumstance, $D$, $U = A - D$. In this relaxed case, utility is the difference between the individuals predicted and actual class. 
\end{definition}

\begin{definition}[Luck-Egalitarian EOP for Supervised Learning]
Predictive model $h$ satisfies luck-egalitarian EOP if for all $\hat{y}\in\hat{Y}=\{0,1\}$ and all $s\in S=\{0,1\}$:
 $$F^h(U \leq u\;|\;S = 0 \cap \hat{Y} = \hat{y}) = F^h(U \leq u \;|\; S = 1 \cap \hat{Y} = \hat{y})$$
where $F^h(U \leq u)$ specifies the distribution of utility across individuals under the predictive model $h$. I.e., difference between the individuals' actual effort, $A$, and their circumstance, $D$, $U = A - D$. In this relaxed case, utility is the difference between the individuals predicted and actual class.
\end{definition}

\subsection{Predicted Outcome}
The predicted outcome family of fairness metrics are the simplest, and most intuitive, notions of fairness. This is because the metrics in this category only use the model's final predictions and do not rely on the original class labels. More explicitly, the predicted outcome class of metrics focuses on using the predicted outcome of various different demographic distributions of the data, and models only satisfy this definition if both the marginalized and non-marginalized groups have equal probability of being assigned to the positive predictive class \cite{verma_fairness_2018}. 

Many different metrics fall into the predicted outcome category, such as group fairness and conditional statistical parity. Additionally, each metric in this group satisfies Rawls' definition of EOP as well as satisfies the statistical constraint of independence.

\subsubsection{Group Fairness}
Group fairness is also often called statistical parity, statistical fairness, equal acceptance rate, demographic parity, or benchmarking. As the name implies, it requires that there is an equal probability for both individuals in the marginalized and non-marginalized groups to be assigned to the positive class \cite{dwork_fairness_2011, kusner_loftus_russell_silva}. Notationally, group fairness can be written as: 
$$P[\hat{Y} = 1 \;|\; S = 0] = P [\hat{Y} = 1 \;|\; S  = 1]$$

Heidari et. al \cite{heidari_moral_2019} derive the proof that group fairness falls into Rawls' definition of EOP and we recall their proof here. 
\begin{proposition}[Group Fairness as Rawls' EOP \cite{heidari_moral_2019}]
Consider the binary classification task where $Y, \hat{Y}=\{0,1\}$. Suppose $U = A - D$, $A = \hat{Y}$, and $D = Y = 1$ (i.e., the effort-based utility of all individuals is assumed to be the same). Then, the conditions of Rawls' EOP is equivalent to group fairness when $\hat{Y} = 1$.
\end{proposition}

\begin{proof}
Recall that Rawls' EOP requires that $s\in S = \{0,1\} $, $y \in Y= \{0,1\}$, $u = a - d\in\{-1, 0\}$:
$$P[U\leq u \;|\; S = 0 \cap Y = y] = P [U \leq u \;|\; S = 1 \cap Y = y]$$
Replacing $U$ with $(A - D)$, $D$ with 1, and $A$ with $\hat{Y}$, the above is equivalent to:
$$P[A - D \leq u \;|\; S = 0 \cap Y = 1] = P[A - D \leq u \;|\; S = 1 \cap Y = 1]$$
$$P[\hat{Y} - 1 \leq u \;|\; S = 0 \cap Y = 1] = P[\hat{Y} - 1 \leq u \;|\; S = 1 \cap Y = 1]$$
$$P[\hat{Y} \leq u + 1\;|\; S = 0] = P[\hat{Y} \leq u + 1 \;|\; S = 1]$$
$$P[\hat{Y} = \hat{y}\;|\; S = 0] = P[\hat{Y} = \hat{y} \;|\; S = 1]$$
because of the facts that $ u = \hat{y} - y$ and $y = 1$ produce the result $\hat{y} = u + 1$. This is equal to the definition for statistical parity when $\hat{Y} = 1$, therefore, the conditions of Rawls' EOP is equivalent to statistical parity.
\end{proof}


One important derivative of group fairness is disparate impact, which we discuss below.

\paragraph{ $\square$ Disparate Impact} It is no coincidence that the fairness metric of disparate impact shares a name with the legal term of disparate impact. This fairness metric, discussed extensively in \cite{feldman_certifying_2015, barocas_big_2016}, was designed to be the mathematical counterpart to the legal notion. It is defined as: 

$$\frac{P[\hat{Y} = 1 \;|\; S = 0]}{P[\hat{Y} = 1 \;|\; S = 1]} \geq 1 - \epsilon$$

where $\epsilon$ is the allowed slack of the metric and is usually set to $0.2$ to achieve the $80\%$ rule of disparate impact law. This equation says that the proportion of positive predictions for both the marginalized and non-marginalized groups must be similar (around threshold $1 - \epsilon$). 

Like group fairness, it also aligns with Rawls' EOP and the statistical notion of independence. To see how it aligns with Rawls' EOP, consider the last line from the proof above (when substituting $\hat{Y}$ with 1): 
$$P[\hat{Y} = 1\;|\; S = 0] = P[\hat{Y} = 1 \;|\; S = 1]$$
Dividing both sides by $P[\hat{Y} = 1 \;|\; S = 1]$ we achieve
\begin{align*}
    \frac{P[\hat{Y} = 1 \;|\; S = 0]}{P[\hat{Y} = 1 \;|\; S = 1]} & = 1\\
    \frac{P[\hat{Y} = 1 \;|\; S = 0]}{P[\hat{Y} = 1 \;|\; S = 1]} & \geq 1 - \epsilon
\end{align*}

which is the same equation for disparate impact. Additionally, it matches exactly with the relaxed notion of independence mentioned in Section 2.1.1.


\subsubsection{Conditional Statistical Parity}
Conditional statistical parity, or conditional group fairness, is an extension of group fairness which allows a certain set of legitimate attributes to be factored into the outcome \cite{corbett-davies_algorithmic_2017}. Factors are considered ``legitimate'' if they can be justified by ethics, by the law, or by a combination of both. This notion of fairness was first defined by Kamiran et al. in 2013 who wanted to quantify explainable and illegal discrimination in automated decision making where one or more attributes could contribute to the explanation \cite{Kamiran:QuantifyingExplainableDiscrimination2013}. Conditional statistical parity is satisfied if both marginalized and non-marginalized groups have an equal probability of being assigned to the positive predicted class when there is a set of legitimate factors that are being controlled for. Notationally, it can be written as: 
$$P[\hat{Y} = 1 \;|\; L_1 = a \cap L_2 = b \cap S = 0] = P[\hat{Y} = 1 \;|\; L_1 = a \cap L_2 = b \cap S = 1]$$
where $L_1, L_2$ are legitimate features that are being conditioned on. For example, if the task was to predict if a certain person makes over \$50,000 a year\footnote{This is the proposed task of a famous toy dataset in fair machine learning called the Adult Income dataset \cite{Dua-2019}.}, then $L_1$ could represent work status and $L_2$ could be the individuals relationship status. Another, simplified way to write this can be seen as: 
$$P[\hat{Y} = 1 \;|\; L = \ell \cap S = 1] = P[\hat{Y} = 1 \;|\; L = \ell \cap S = 0]$$
where $\ell\in L$ is the set of legitimate features being conditioned on.

Like group fairness, conditional statistical parity also falls under the definition of Rawls' EOP and adheres to a relaxed notion of statistical independence. The proof for conditional statistical parity falling into Rawls' EOP can be seen when we let $D$ be equal to $L$. In other words, instead of having the effort-based utility be equal to the true label of the individual, let the effort-based utility be equal to the set of legitimate factors that affect the result. Then, the proof reduces to the proof shown for group fairness. This substitution is allowed since the legitimate factors provide reasoning for the actual label $Y$.



Furthermore, conditional statistical parity has a theoretical grounding of Simpson's paradox as it tries to incorporate extra conditioning information beyond the original class label. Simpson's paradox says that if a correlation occurs in several different groups, it may disappear, or even reverse, when the groups are aggregated \cite{blitzstein2019}. This event can be seen in Fig. \ref{fig:simpsons}.
Put mathematically, Simpson's paradox can be written as:
$$P[A  \;|\; B \cap C] < P [A \;|\; B^c \cap C] \textrm{ and } P[A  \;|\; B \cap C^c] < P [A \;|\; B^c \cap C^c]$$
$$\textrm{but}$$
$$P[A\;|\;B] > P[A\;|\;B^c]$$

An analysis that does not consider all of the relevant statistics might suggest that unfairness and discrimination is at play, when in reality, the situation may be morally and legally acceptable if all of the information was known.

\begin{figure}[t!]
    \centering
    \includegraphics[scale=.4]{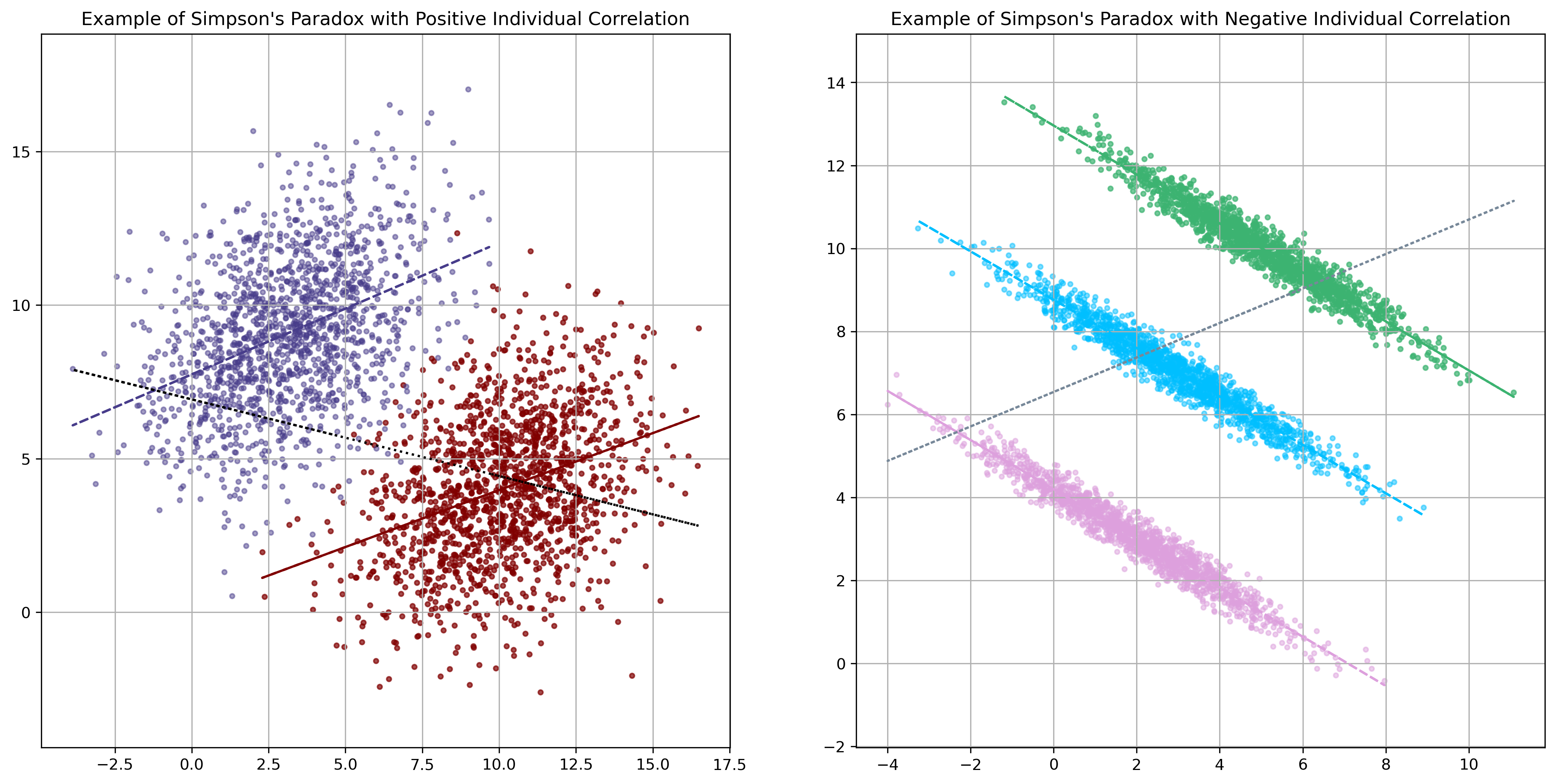}
    \caption{Two different examples of Simpson's paradox. In the image on the left, while both the blue and red groups are positively correlated among themselves, when taken as a whole, there is a negative correlation. The opposite can be seen in the second image where the groups are negatively correlated when taken alone, but when aggregated, a positive correlation is produced.}
    \label{fig:simpsons}
\end{figure}

\subsection{Predicted and Actual Outcome}
The predicted and actual outcome class of metrics goes one step beyond the predicted outcome class since it uses both the model's predictions as well as the true labels of the data. This class of fairness metrics includes: predictive parity, false positive error rate balance, false negative error rate balance, equalized odds, conditional use accuracy, overall accuracy equality, and treatment equality.

\subsubsection{Predictive Parity}
Predictive parity, otherwise known by the name outcome test, is a fairness metric that requires the positive predictive values to be similar across both marginalized and non-marginalized groups \cite{chouldechova2016fair}. Mathematically, it can be seen in two ways:
$$P[Y=1 \;|\; \hat{Y} = 1\cap S = 0] = P[Y=1\;|\;\hat{Y} = 1 \cap S = 1]$$
$$P[Y=0 \;|\; \hat{Y} = 1\cap S = 0] = P[Y=0\;|\;\hat{Y} = 1 \cap S = 1]$$
since if a classifier has equal positive predictive values for both groups, it will also have equal false discovery rates. Since predictive parity is simply conditional use accuracy when $\hat{Y} = 1$, it falls into the same philosophical category as conditional use accuracy, which is luck-egalitarian EOP. This proof can be seen in Section 5.2.5. Further, \cite{barocas-hardt-narayanan} states that predictive parity aligns with the relaxed version of sufficiency.

\subsubsection{False Positive Error Rate Balance}
False positive error rate balance, otherwise known as predictive equality, requires that false positive rates are similar across different groups \cite{chouldechova2016fair}. It can be seen mathematically through the following two equations 

$$P[\hat{Y}=1 \;|\; Y = 0 \cap S = 0] = P[\hat{Y} = 1 \;|\; Y = 0 \cap S = 1]$$
$$P[\hat{Y}=0 \;|\; Y = 0 \cap S = 0] = P[\hat{Y} = 0 \;|\; Y = 0 \cap S = 1]$$
It can be written in two different forms because if a classifier has equal false positive rates for both groups, it will also have equal true negative rates. As opposed to predictive parity, false positive error rate balance falls into Rawls' EOP for its philosophical classification. This is because it is a constrained version of equalized odds, which is classified as Rawls' EOP as seen in Section 5.2.4. Additionally, false positive error rate balance aligns with the relaxed version of separation \cite{barocas-hardt-narayanan}.

\subsubsection{False Negative Error Rate Balance}
False negative error rate balance, also called equal opportunity, is the direct opposite of the above fairness metric of false positive error rate balance in that it requires false negative rates to be similar across different groups \cite{chouldechova2016fair}. This metric can also be framed two different ways: 
$$P[\hat{Y}=0 \;|\; Y = 1 \cap S = 0] = P[\hat{Y} = 0 \;|\; Y = 1 \cap S = 1]$$
$$P[\hat{Y}=1 \;|\; Y = 1 \cap S = 0] = P[\hat{Y} = 1 \;|\; Y = 1 \cap S = 1]$$
since a classifier that has equal false negative rates across the two groups will also have equal true positive rates. Similar to false positive error rate balance, false negative error rate balance can be classified into Rawls' EOP as well since it is also a constrained version of equalized odds. Additionally, like false positive error rate balance, false negative error rate balance aligns with the relaxed version of separation \cite{barocas-hardt-narayanan}.

\subsubsection{Equalized Odds}
The fairness metric of equalized odds is also known as conditional procedure accuracy equality and disparate mistreatment. It can be seen as the combination of false positive error rate balance and false negative error rate balance since it requires that true and false positive rates are similar across different groups \cite{moritz_google_price_srebro}. 
$$P[\hat{Y} = 1 \;|\; Y = y \cap S = 0] = P[\hat{Y} = 1 \;|\; Y = y \cap S = 1] \;\; \textrm{ for } \;\; y\in\{0,1\}$$

As mentioned in the above two metrics of false negative/positive error rate balance, equalized odds aligns with Rawls' EOP. \cite{heidari_moral_2019} provides a proof for this classification, which we recall here: 
\begin{proposition}[Equalized Odds as Rawls' EOP \cite{heidari_moral_2019}]
Consider the binary classification task where $Y, \hat{Y} = \{0,1\} $. Suppose $U = A - D$, $A = \hat{Y}$ (i.e., the actual utility is equal to the predicted label) and $D = Y$ (i.e., effort-based utility of an individual is assumed to be the same as their true label). Then the conditions of Rawls' EOP is equivalent to those of equalized odds. 
\begin{proof}
Recall that Rawls' EOP requires that for $s \in S = \{0,1\}$, $y \in Y = \{0,1\} $, and $ u \in \{-1, 1\}$:
 $$P[U \leq u\;|\;Y = y \cap S = 0] = P[U \leq u \;|\; Y = y \cap S = 1]$$
 By replacing $U$ with $A - D$, $D$ with $Y$, and $A$ with $\hat{Y}$, the above is equivalent to: 
 $$P[A - D \leq u\;|\;Y = y \cap S = 0] = P[A - D \leq u \;|\; Y = y \cap S = 1]$$
 $$P[\hat{Y} - Y \leq u\;|\;Y = y \cap S = 0] = P[\hat{Y} - Y \leq u \;|\; Y = y \cap S = 1]$$
  $$P[\hat{Y} \leq u + y\;|\;Y = y \cap S = 0] = P[\hat{Y} \leq u + y\;|\; Y = y \cap S = 1]$$
  $$P[\hat{Y} =\hat{y} \;|\;Y = y \cap S = 0] = P[\hat{Y} = \hat{y} \;|\; Y = y \cap S = 1]$$
  since $u = a - d = \hat{y} - y$ produces $\hat{y} = u + y$. Then, the last line is identical to the conditions of equalized odds for binary classification when $\hat{Y} = 1$.
\end{proof}

\end{proposition}
Further, since equalized odds is a combination of false negative error rate balance and false positive error rate balance, it also aligns with separation. But, instead of being the relaxed version, equalized odds aligns exactly with separation. 

\subsubsection{Conditional Use Accuracy}
Conditional use accuracy, also termed as predictive value parity, requires that positive and negative predicted values are similar across different groups \cite{fairness_in_criminal_justice}. Statistically, it aligns exactly with the requirement for sufficiency \cite{barocas-hardt-narayanan}. Mathematically, it can be written as follows:
$$P[Y = y \;|\; \hat{Y} = y \cap S = 0] = P[Y = y \;|\; \hat{Y} = y \cap S = 1] \;\; \textrm{ for } \;\; y\in\{0,1\}$$
In \cite{heidari_moral_2019}, they provide a proof that conditional use accuracy falls into the Luck-Egalitarian EOP criterion and we recall their proof here: 

\begin{proposition}[Conditional Use Accuracy as Luck-Egalitarian EOP \cite{heidari_moral_2019}]
Consider the binary classification task where $y \in Y = \{0,1\} $. Suppose that $U = A - D$, $A = Y$, and $D = \hat{Y}$ (i.e., the effort-based utility of an individual under model a $h$ is assumed to be the same as their predicted label). Then the conditions of Luck-Egalitarian EOP are equivalent to those of conditional use accuracy (otherwise known as predictive value parity). 
\begin{proof}
Recall that Luck-Egalitarian EOP requires that for $s\in S = \{0,1\}$, $\hat{y}\in\hat{Y}=\{0,1\}$, and $u \in \{-1, 1\}$:
$$P[U \leq u \;|\; \hat{Y}=\hat{y} \cap S = 0 ] = P[U \leq u \;|\; \hat{Y}=\hat{y} \cap S = 1 ]$$

Replacing $U$ with $A - D$, $D$ with $\hat{Y}$, and $A$ with $Y$, we obtain the following:
$$P[A - D \leq u \;|\; \hat{Y} = \hat{y} \cap S = 0 ] = P[A - D \leq u \;|\; \hat{Y}=\hat{y} \cap S = 1 ]$$
$$P[Y - \hat{Y} \leq u \;|\; \hat{Y} = \hat{y} \cap S = 0 ] = P[Y - \hat{Y} \leq u \;|\; \hat{Y}=\hat{y} \cap S = 1 ]$$
$$P[Y \leq u + \hat{y}\;|\; \hat{Y} = \hat{y} \cap S = 0 ] = P[Y \leq u + \hat{y}\;|\; \hat{Y}=\hat{y} \cap S = 1 ]$$
$$P[Y = y \;|\; \hat{Y} = \hat{y} \cap S = 0 ] = P[Y = y \;|\; \hat{Y}=\hat{y} \cap S = 1 ]$$
since $u = a - d = y - \hat{y}$ produces the result that $y = u + \hat{y}$. The last line is then equal to the statement for conditional use accuracy when $y = \hat{y} = \{0,1\}$.
\end{proof}
\end{proposition}

\subsubsection{Overall Accuracy Equality}
As the name implies, overall accuracy equality requires similar prediction accuracy across different groups. In this case, we are assuming that obtaining a true negative is as desirable as obtaining a true positive \cite{fairness_in_criminal_justice}. According to \cite{barocas-hardt-narayanan}, it matches exactly with the statistical measure of independence. Mathematically, it can be written as:
$$P[\hat{Y}=y \;|\; Y = y \cap S = 0] = P[\hat{Y} = y\;|\; Y = y \cap S = 1] \;\; \textrm{ for } \;\; y \in \{0,1\}$$

Overall accuracy equality is the third fairness metric that Heidari et al. proves belongs to the Rawls' EOP category of fairness metrics \cite{heidari_moral_2019}. We recall their proof here: 
\begin{proposition}[Overall Accuracy Equality as Rawls' EOP \cite{heidari_moral_2019}]
Consider the binary classification task where $Y, \hat{Y} = \{0,1\}$. Suppose $U = A - D$ where $A = (\hat{Y} - Y) ^ 2$ and $D = 0$ (i.e., effort-based utility of all individuals are assumed to be the same and equal to 0). Then the conditions of Rawl's EOP is equivalent to overall accuracy equality. 
\begin{proof}
Recall that Rawls' EOP requires that for $s \in S = \{0,1\}$, $y \in Y= \{0,1\} $, and $u \in \{0,1\}$:
 $$P[U \leq u\;|\; Y = y \cap S = 0] = P[U \leq u \;|\; Y = y \cap S = 1]$$
 By replacing $U$ with $A - D$, $D$ with $0$, and $A$ with $(\hat{Y} - Y) ^ 2$, the above is equivalent to: 
 $$P[A - D \leq u \;|\; Y = y \cap S = 0] = P[A - D \leq u \;|\; Y = y \cap S = 1]$$
  $$P[(\hat{Y} - Y)^2 \leq u \;|\; Y = y \cap S = 0] = P[(\hat{Y} - Y)^2 \leq u \;|\; Y = y \cap S = 1]$$
  $$P[\hat{Y}^2-2\hat{Y}Y + Y^2 \leq u \;|\; Y = y \cap S = 0] = P[\hat{Y}^2-2\hat{Y}Y + Y^2 \leq u \;|\; Y = y \cap S = 1] $$
   $$P[\hat{Y}^2 \leq u + 2\hat{y}y - y^2\;|\; Y = y \cap S = 0] = P[\hat{Y}^2 \leq u +2\hat{y}y - y^2\;|\; Y = y \cap S = 1] $$
   $$P[\hat{Y}^2 = \hat{y}^2\;|\; Y = y \cap S = 0] = P[\hat{Y}^2  = \hat{y}^2\;|\; Y = y \cap S = 1] $$
   $$P[\hat{Y} = \hat{y}\;|\; Y = y \cap S = 0] = P[\hat{Y}  = \hat{y}\;|\; Y = y \cap S = 1] $$
   
   since $u = \hat{y}^2 -2\hat{y}y+y^2$ and $\hat{y}^2 = u + 2\hat{y}y - y^2$. This is equal to the statement for overall accuracy equality when $Y = \hat{Y} = \{0,1\}$.
\end{proof}
\end{proposition}

\subsubsection{Treatment Equality}
Treatment equality analyzes fairness by looking at how many errors were obtained rather than through the lens of accuracy. It requires an equal ratio of false negative and false positive values for all groups \cite{fairness_in_criminal_justice}. Further, it agrees exactly with the statistical measure of separation \cite{barocas-hardt-narayanan}. 

$$\frac{FN_{S = 0}}{FP_{S = 0}} = \frac{FN_{S = 1}}{FP_{S = 1}}$$

Treatment equality can be considered as the ratio of false positive error rate balance and false negative error rate balance. Since both of these metrics fall into the Rawls' EOP category, treatment equality does as well. 

\subsection{Predicted Probability and Actual Outcome}
The predicted probability and actual outcome category of fairness metrics is similar to the above category of metrics that use the predicted and actual outcomes. But, instead of using the predictions themselves, this category uses the probability of being predicted to a certain class. This category of metrics includes: test fairness, well calibration, balance for the positive class, and balance for the negative class. The first two metrics fall in line with the statistical measure of sufficiency while the last two align with the relaxed notion of separation.

\subsubsection{Test Fairness}
Test fairness is satisfied if, for any predicted probability score $p \in \mathcal{P}$, subjects in both the marginalized and non-marginalized groups have equal probability of actually belonging to the positive class. Test fairness has also been referenced by the terms calibration, equal calibration, and matching conditional frequencies \cite{chouldechova2016fair}. Mathematically, it can be written as follows:

$$P[Y = 1 \;|\; \mathcal{P} = p \cap S = 0] = P[Y = 1 \;|\; \mathcal{P} = p \cap S = 1]$$

Test fairness falls under the luck-egalitarian EOP category. To see why this is the case, we recall the last line of the proof for conditional use accuracy:
$$P[Y = y \;|\; S = 0 \cap \hat{Y} = \hat{y}] = P[Y = y \;|\; S = 1 \cap \hat{Y}=\hat{y}]$$
By letting $Y = 1$ and replacing the predicted class $\hat{y}\in\hat{Y}$ with the probability to be assigned to a class $p \in \mathcal{P}$, we receive the following: 
$$P[Y = 1 \;|\; S = 0 \cap \mathcal{P} = p] = P[Y = 1 \;|\; S = 1 \cap \mathcal{P}=p]$$
which is the same as the condition for test fairness. 

\subsubsection{Well Calibration}
Well calibration is very similar to the metric of test fairness, but it goes one step further and it requires that for any predicted probability score $p \in \mathcal{P}$, not only should the majority and minority classes have equal probability of belonging to the positive class, but this probability should be $p$ \cite{kleinberg2016inherent}.
$$P[Y=1\;|\;\mathcal{P}=p \cap S = 0] = P[Y=1\;|\;\mathcal{P}=p \cap S = 1] = p$$
Since well calibration is an extension of test fairness, it in turn also falls under the luck-egalitarian EOP category. 

\subsubsection{Balance for the Positive Class}
As the name suggests, the balance for the positive class metric requires that individuals who experience a positive outcome, regardless of which group they belong to, should have an equal mean predicted probability of being in the positive class \cite{kleinberg2016inherent}. It can be seen as being similar to the metric of equal opportunity, which says that a classifier should give equivalent treatment to all groups. 
$$\mathbb{E}[\mathcal{P} \;|\; Y = 1 \cap S = 0] =\mathbb{E}[\mathcal{P} \;|\; Y = 1 \cap S = 1]$$
Like false positive error rate balance, the balance for the positive class metric can be seen as a derivative of the equalized odds metric when $Y = 1$. Additionally, instead of taking into account the predicted label $\hat{y} \in \hat{Y}$, it concerns itself with the predicted probability $\mathcal{P}$. Since equalized odds falls into Rawls' EOP category of metrics, the balance for the positive class metric does as well.

\subsubsection{Balance for the Negative Class}
The metric of balance for the negative class is opposite of the balance for the positive class metric. Instead of requiring balance in the predictive mean of the positive class, it requires balance in the predicted mean of the negative class \cite{kleinberg2016inherent}. It is similar to the measure of false positive error rate balance. 
$$\mathbb{E}[\mathcal{P} \;|\; Y = 0 \cap S = 0] =\mathbb{E}[\mathcal{P} \;|\; Y = 0 \cap S = 1]$$
Same as the argument for balance for the positive class, the balance for the negative class metric is a derivative of equalized odds when $Y = 0$ and we approximate $\hat{y} \in \hat{Y}$ with probability score $\mathcal{P}$. Therefore, the balance for the negative class metric also falls under Rawls' EOP category.

\begin{table}[t!]
\centering
\label{tab:big-table}
\setlength{\extrarowheight}{0pt}
\addtolength{\extrarowheight}{\aboverulesep}
\addtolength{\extrarowheight}{\belowrulesep}
\setlength{\aboverulesep}{0pt}
\setlength{\belowrulesep}{0pt}
\caption{Definitions and classifications for popular fair machine learning classification metrics. IND = Independent, SUF = Sufficiency, SEP = Separation, R = Relaxed, E = Equivalent, LE = Luck-Egalitarian}
\resizebox{\linewidth}{!}{%
\begin{tabular}{llcccc} 
\toprule
\rowcolor[rgb]{0.592,0.792,0.792} \multicolumn{1}{c}{\textcolor[rgb]{0,0.502,0.502}{\textbf{Metric}}} & \multicolumn{1}{c}{\textcolor[rgb]{0,0.502,0.502}{\textbf{Definition}}}                                                                                        & \textcolor[rgb]{0,0.502,0.502}{\textbf{Formula}}                                                                                                                                                           & \textcolor[rgb]{0,0.502,0.502}{\textbf{Closest}} & \textcolor[rgb]{0,0.502,0.502}{\textbf{R/E}} & \textbf{\textcolor[rgb]{0,0.502,0.502}{Rawls' / LE}}  \\ 
\hline
                                                                                                      &                                                                                                                                                                &                                                                                                                                                                                                            &                                                  &                                              &                                                      \\ 
\hline
\multicolumn{6}{l}{{\cellcolor[rgb]{0.792,0.894,0.894}}\textcolor[rgb]{0,0.502,0.502}{\textbf{1. Predicted Outcome}}}                                                                                                                                                                                                                                                                                                                                                                                                                                                                                                                        \\ 
\hline
\begin{tabular}[c]{@{}l@{}}Group Fairness \\(Statistical Parity)\end{tabular}                         & \begin{tabular}[c]{@{}l@{}}All groups have equal probability of being\\assigned to the positive class.\end{tabular}                                            & $P[\hat{Y}=1 \;|\; S=0] = P[\hat{Y}=1 \;|\; S=1]$                                                                                                                                                          & IND                                              & E                                            & Rawls'                                                \\ 
\hline
-- Disparate Impact                                                                                   & \begin{tabular}[c]{@{}l@{}}Proportion of positive predictions \\of all groups must be similar.\end{tabular}                                                    & $\frac{P[\hat{Y}=1\;|\;S=0]}{P[\hat{Y}=1 \;|\; S = 1]} \geq 1 - \epsilon$                                                                                                                                  & IND                                              & R                                            & Rawls'                                                \\ 
\hline
Conditional Statistical Parity                                                                        & \begin{tabular}[c]{@{}l@{}}Requires statistics for all groups to \\be equal, allowing for a set of legitimate\\factors $L = \ell$.~\end{tabular}               & $P[\hat{Y}=1 \;|\; L = \ell \cap S = 0] = P[\hat{Y}=1 \;|\; L = \ell \cap S = 1]$                                                                                                                          & IND                                              & R                                            & Rawls'                                                \\ 
\hline
                                                                                                      &                                                                                                                                                                &                                                                                                                                                                                                            &                                                  &                                              &                                                      \\ 
\hline
\multicolumn{6}{l}{{\cellcolor[rgb]{0.792,0.894,0.894}}\textcolor[rgb]{0,0.502,0.502}{\textbf{2. Predicted and Actual Outcome}}}                                                                                                                                                                                                                                                                                                                                                                                                                                                                                                             \\ 
\hline
Predictive Parity                                                                                     & \begin{tabular}[c]{@{}l@{}}Similar positive~predictive values\\(or FDR) across groups.\end{tabular}                                                            & \begin{tabular}[c]{@{}c@{}}PPV: $P[Y = 1 \;|\; \hat{Y}= 1 \cap S = 0] = P[Y = 1 \;|\; \hat{Y}=1 \cap S =1]$\\FDR:$P[Y = 0 \;|\; \hat{Y}= 1 \cap S = 0] = P[Y = 0 \;|\;\hat{Y}=1 \cap S =1]$\end{tabular}   & SUF                                              & R                                            & LE                                                   \\ 
\hline
False Positive Error Rate Balance                                                                     & \begin{tabular}[c]{@{}l@{}}Similar false positive rates \\(or TNR) across groups.\end{tabular}                                                                 & \begin{tabular}[c]{@{}c@{}}FPR:$P[\hat{Y} = 1 \;|\; Y= 0 \cap S = 0] = P[\hat{Y} = 1 \;|\; Y =0 \cap S =1]$\\TNR:$P[\hat{Y} = 0 \;|\; Y= 0 \cap S = 0] = P[\hat{Y} = 0 \;|\; Y =0 \cap S =1]$\end{tabular} & SEP                                              & R                                            & Rawls'                                                \\ 
\hline
False Negative Error Rate Balance                                                                     & \begin{tabular}[c]{@{}l@{}}Similar false negative rates \\(or TPR) across groups.\end{tabular}                                                                 & \begin{tabular}[c]{@{}c@{}}FNR:$P[\hat{Y} = 0 \;|\; Y= 1 \cap S = 0] = P[\hat{Y} = 0 \;|\; Y =1 \cap S =1]$\\TPR:$P[\hat{Y} = 1 \;|\; Y= 1 \cap S = 0] = P[\hat{Y} = 1 \;|\; Y =1 \cap S =1]$\end{tabular} & SEP                                              & R                                            & Rawls'                                                \\ 
\hline
Equalized Odds                                                                                        & \begin{tabular}[c]{@{}l@{}}Similar false positive and true positive\\rates across groups.\end{tabular}                                                         & \begin{tabular}[c]{@{}c@{}}$P[\hat{Y}=1 \;|\; Y = y \cap S = 0 ] = P[\hat{Y}=1\;|\;Y=y \cap S =1] $\\for $y \in \{0, 1\}$\end{tabular}                                                                     & SEP                                              & E                                            & Rawls'                                                \\ 
\hline
Conditional Use Accuracy                                                                              & \begin{tabular}[c]{@{}l@{}}Similar positive and negative predictive\\values across groups.\end{tabular}                                                        & \begin{tabular}[c]{@{}c@{}}$P[Y=y \;|\;\hat{Y}=y\cap S = 0] = P[Y=y\;|\;\hat{Y}=y\cap S=1]$\\for~$y \in \{0, 1\}$\end{tabular}                                                                             & SUF                                              & E                                            & LE                                                   \\ 
\hline
Overall Accuracy Equality                                                                             & Requires similar accuracy across groups.                                                                                                                       & \begin{tabular}[c]{@{}c@{}}$P[\hat{Y}=y \;|\;Y=y\cap S = 0] = P[\hat{Y}=y\;|\;Y=y\cap S=1]$\\for~$y \in \{0, 1\}$\end{tabular}                                                                             & IND                                              & E                                            & Rawls'                                                \\ 
\hline
Treatment Equality                                                                                    & \begin{tabular}[c]{@{}l@{}}Equal ratio of false negatives and false\\positive between groups.\end{tabular}                                                     & $\frac{FN_{S=0}}{FP_{S=0}} = \frac{FN_{S=1}}{FP_{S=1}}$                                                                                                                                                    & SEP                                              & E                                            & Rawls'                                                \\ 
\hline
                                                                                                      &                                                                                                                                                                &                                                                                                                                                                                                            &                                                  &                                              &                                                      \\ 
\hline
\multicolumn{6}{l}{{\cellcolor[rgb]{0.792,0.894,0.894}}\textbf{\textcolor[rgb]{0,0.502,0.502}{3. Predicted Probabilities and Actual Outcomes}}}                                                                                                                                                                                                                                                                                                                                                                                                                                                                                              \\ 
\hline
Test Fairness                                                                                         & \begin{tabular}[c]{@{}l@{}}All groups have equal probability to belong to\\the positive class.\end{tabular}                                                    & \begin{tabular}[c]{@{}c@{}}$P[Y = 1 \;|\; \mathcal{P}=p\cap S = 0] = P [Y=1\;|\;\mathcal{P}=p\cap S=1]$\\for any predicted probability score $p \in \mathcal{P}$\end{tabular}                              & SUF                                              & R                                            & LE                                                   \\ 
\hline
Well Calibration                                                                                      & \begin{tabular}[c]{@{}l@{}}The probability of all groups to belong to the\\positive class is the predicted probability\\score~$p\in\mathcal{P}$ .\end{tabular} & $P[Y = 1 \;|\; \mathcal{P}=p\cap S = 0] = P [Y=1\;|\;\mathcal{P}=p\cap S=1] = p$                                                                                                                           & SUF                                              & E                                            & LE                                                   \\ 
\hline
Balance for Positive Class                                                                            & \begin{tabular}[c]{@{}l@{}}Equal mean predicted probabilities for all people\\in the positive class, regardless of group.\end{tabular}                         & $\mathbb{E}[\mathcal{P} \;|\; Y = 1 \cap S = 0 ] = \mathbb{E}[\mathcal{P} \;|\;Y=1\cap S=1]$                                                                                                               & SEP                                              & R                                            & Rawls'                                                \\ 
\hline
Balance for Negative Class                                                                            & \begin{tabular}[c]{@{}l@{}}Equal mean predicted probabilities for all subjects\\in the negative class, regardless of group.\end{tabular}                       & $\mathbb{E}[\mathcal{P} \;|\; Y = 0 \cap S = 0 ] = \mathbb{E}[\mathcal{P} \;|\;Y=0\cap S=1]$                                                                                                               & SEP                                              & R                                            & Rawls'                                                \\
\bottomrule
\end{tabular}
}
\end{table}

\subsection{Individual Fairness Notions}
Up until this point, the metrics we have discussed all focus on the notion of \textit{group fairness}. In other words, these metrics ensure some kind of statistical parity for members of different groups and not a specific individual \cite{binns2019apparent}. Another set of fairness metrics that consider the fairness as it relates to a specific individual is called \textit{individual fairness}. Individual fairness ensures that people who are similar in the eyes of the classification task are treated similarly (i.e., obtain the same prediction) \cite{binns2019apparent}. In this section, we recount the first (and most famous) notion of individual fairness - fairness through awareness. We note that other individual fairness metrics exist, such as \cite{joseph2016fairness, jung2020algorithmic, lahotiifair, pmlr-v28-zemel13}, and we direct interested readers to these publications, as well as a survey over them \cite{binns2019apparent}, for more detail.

\subsubsection{\textbf{Fairness through Awareness}} 
Fairness through awareness, most commonly called individual fairness, was first proposed by Dwork et al. in 2012 \cite{dwork_fairness_2011}. The motivation in creating fairness through awareness was that there was concerns that simply using statistical parity between different groups could result in unfair outcomes at the individual level. To solve this issue, Dwork et al. proposed to use a distance metric that measured how similar an individual was to another. Two individuals were considered alike if their combination of task-relevant attributes were nearby each other, and the overall process was deemed fair if the two individuals (who were alike) received the same outcome from the model \cite{binns2019apparent}. This process can be seen as being similar to the legal practice of situation testing. Situation testing is an experimental method that aims to establish discrimination on the spot \cite{benedick207situation}. It takes pairs of individuals who are similar, but do not necessarily belong to the same group, and puts them in the same situation. If the individual who is part of the marginalized group is treated differently than the individual in the non-marginalized group, then there is a viable case of discrimination that can be brought to court. Several research works \cite{luong2011knn, zhang2016situation} studied the use of kNN and causal Bayesian networks to facilitate the similarity measurements in situation testing based discrimination detection. Additionally, fairness through awareness aligns with Aristotle's conception of ``justice as consistency'' \cite{binns2019apparent, winston_1974}.

A downfall of this metric is that it does not allow for comparison of \textit{all} individuals since it only compares \textit{similar} individuals. So in the hiring example, the applicants who have similar background experiences can be compared to each other, but they cannot be compared to those who have different prior work experience. This makes it impossible to construct a total ranking of all the candidates. Additionally, fairness through awareness can be difficult to implement as it requires explicitly defining what similarity means in a certain context and what is considered similar in one case may not be considered similar in another. Further, specifically for fairness through awareness, it requires the defining of a distance metric by the people who set the policy, which is not a simple task to do \cite{binns2019apparent}.   

\subsubsection{Group Fairness vs. Individual Fairness}
Many technical research papers assume that both group and individual fairness are important, but conflicting, measures \cite{lahotiifair, pmlr-v28-zemel13}. But, Binns argues that this conflict is based on a misconception, and when we look at the philosophical underpinnings of group and individual fairness, they are not actually trying to achieve different things \cite{binns2019apparent}. Specifically, in \cite{binns2019apparent} he wrote: 
\begin{displayquote}
``... while there may be natural affinities between the usual formulations of individual fairness and [Aristotle's] consistency, and the usual motivations for group fairness and egalitarianism, these are only surface deep. Consistency and egalitarianism themselves do not conflict at the level of principle. In fact they can even be seen as mutually implied in so far as luck egalitarianism aims to remove luck from allocation, it implies consistency. And in so far as the process for defining task-relevant similarity can already be an exercise in normative judgement, it makes sense to incorporate egalitarian concerns into it. It is therefore both possible and indeed coherent to adopt consistency-respecting formulations of group fairness, as well as egalitarian formulations of individual fairness.'' 
\end{displayquote}
All this being to say that the philosophical ideas of individual and group fairness are actually the same. Binns goes on to reiterate that: ``the appearance of conflict between the two is an artifact of the failure to fully articulate assumptions behind them, and the reasons for applying them in a particular context'' \cite{binns2019apparent}.

\section{Causal-Based Fairness Notions}
\label{sec:cfair}

In Section \ref{sec:sfairness}, we presented fairness notions in machine learning that are based on various statistical discrimination criteria. These criteria are observational as they depend only on the joint distribution of predictor (algorithm), marginalization attribute, features, and the final outcome. Observational criteria have severe inherent limitations that prevent them from resolving matters of fairness in a conclusive manner, mainly because fairness cannot be well assessed based only on mere correlation or association (rung one of Pearl's Ladder of Causation). Additionally, statistics-based fairness metrics do not align well with the legal process of discrimination handling since discrimination claims usually require plaintiffs to demonstrate a causal connection between the challenged decision and the sensitive feature. Most recent fairness notions are causal-based and reflect the now widely accepted idea that using causality is necessary to appropriately address the problem of fairness. In this section, we present fairness in the language of causal reasoning. Causal-based fairness notions differ from the previously explained ones in that they are not totally based on data\footnote{``Data is profoundly dumb. Data can tell you that people who took a medicine recovered faster than those who did not take it, but they can't tell you why." - Judea Pearl \cite{PearlMackenzie18}}, but consider additional knowledge about the structure of the world, in the form of a causal model. 

Causal-based fairness notions are developed mainly under the two causal frameworks: the structural causal model (SCMs) and the potential outcome. SCMs assume that we know the complete causal graph, and hence, we are able to study the causal effect of any variable along many different paths. The potential outcome framework does not assume the availability of the causal graph and instead focuses on estimating the causal effects of treatment variables. In Table \ref{tab:classification}, we present the causal framework to which each causal-based fairness notion discussed in this section belongs.

In \cite{pearl2019seven}, Pearl presented the causal hierarchy through the Ladder of Causation, as shown in Fig. \ref{fig:ladderofcause}. The Ladder of Causation has the 3 rungs: association, intervention, and counterfactual. On the first rung, associations can be inferred directly from the observed data using conditional probabilities and conditional expectations. The intervention rung involves not only seeing what is, but also changing what we see. Interventional questions deal with $P(y|do(x), z)$ which stands for ``the probability of $Y=y$, given that we intervene and set the values of $X$ to $x$ and subsequently observe event $Z=z$.'' Interventional questions cannot be answered from purely observational data alone. They can be estimated experimentally from randomized trials or analytically using causal Bayesian networks. The top rung invokes counterfactuals and deals with $P(y_x|x', y')$ which stands for ``the probability that event $Y=y$ would be observed had $X$ been $x$, given that we actually observed $X$ to be $x'$ and $Y$ to be $y'$.'' Such questions can be computed only when the model is based on functional relations or is structural. In Table \ref{tab:classification}, we also show the causal hierarchical level that each causal-based fairness notion aligns with.

\begin{table}[ht!]\small
\centering
\setlength{\extrarowheight}{0pt}
\addtolength{\extrarowheight}{\aboverulesep}
\addtolength{\extrarowheight}{\belowrulesep}
\setlength{\aboverulesep}{0pt}
\setlength{\belowrulesep}{0pt}
\caption{Classification of causal-based fairness notions. SCM = structure causal model, PO = potential outcome. The last column describes whether the fairness notion involves both $Y$ and $\hat{Y}$ in their counterfactual quantity.}
\resizebox{\linewidth}{!}{%
\begin{tabular}{lcccccc} 
\toprule
\rowcolor[rgb]{0.592,0.792,0.792} \multicolumn{1}{c}{\textcolor[rgb]{0,0.502,0.502}{\textbf{Notion}}} & \textcolor[rgb]{0,0.502,0.502}{\textbf{Association}} & \textcolor[rgb]{0,0.502,0.502}{\textbf{SCM}} & \textcolor[rgb]{0,0.502,0.502}{\textbf{PO}} & \textcolor[rgb]{0,0.502,0.502}{\textbf{Intervention}} & \textcolor[rgb]{0,0.502,0.502}{\textbf{Counterfactual}} & \textcolor[rgb]{0,0.502,0.502}{\textbf{$Y$ and~$\hat{Y}$}}  \\ 
\hline
Total Variation                                                                                      & $\checkmark$                                         &                                              &                                             &                                                       &                                                         &                                                             \\
Total Causal Fairness                                                                                 &                                                      & $\checkmark$                                 &                                             & $\checkmark$                                          &                                                         &                                                             \\
Natural Direct Effect                                                                                 &                                                      & $\checkmark$                                 &                                             & $\checkmark$                                          &                                                         &                                                             \\
Natural Indirect Effect                                                                               &                                                      & $\checkmark$                                 &                                             & $\checkmark$                                          &                                                         &                                                             \\
Path-specific Causal Fairness                                                                         &                                                      & $\checkmark$                                 &                                             & $\checkmark$                                          &                                                         &                                                             \\
Direct Causal Fairness                                                                                &                                                      & $\checkmark$                                 &                                             & $\checkmark$                                          &                                                         &                                                             \\
Indirect Causal Fairness                                                                              &                                                      & $\checkmark$                                 &                                             & $\checkmark$                                          &                                                         &                                                             \\
Counterfactual Fairness                                                                               &                                                      & $\checkmark$                                 &                                             &                                                       & $\checkmark$                                            &                                                             \\
Counterfactual Direct Effect                                                                          &                                                      & $\checkmark$                                 &                                             &                                                       & $\checkmark$                                            &                                                             \\
Counterfactual Indirect Effect                                                                        &                                                      & $\checkmark$                                 &                                             &                                                       & $\checkmark$                                            &                                                             \\
Path-specific Counterfactual Fairness                                                                 &                                                      & $\checkmark$                                 &                                             &                                                       & $\checkmark$                                            &                                                             \\
Proxy Fairness                                                                                        &                                                      & $\checkmark$                                 &                                             & $\checkmark$                                          &                                                         &                                                             \\
Justifiable Fairness                                                                                  &                                                      & $\checkmark$                                 &                                             & $\checkmark$                                          &                                                         &                                                             \\
Counterfactual Direct Error Rate                                                                      &                                                      & $\checkmark$                                 &                                             &                                                       & $\checkmark$                                            & $\checkmark$                                                \\
Counterfactual Indirect Error Rate                                                                    &                                                      & $\checkmark$                                 &                                             &                                                       & $\checkmark$                                            & $\checkmark$                                                \\
Individual Equalized Counterfactual Odds                                                              &                                                      & $\checkmark$                                 &                                             &                                                       & $\checkmark$                                            & $\checkmark$                                                \\
Fair on Average Causal Effect                                                                         &                                                      &                                              & $\checkmark$                                & $\checkmark$                                          &                                                         &                                                             \\
Fair on Average Causal Effect on the Treated~                                                         &                                                      &                                              & $\checkmark$                                &                                                       & $\checkmark$                                            &                                                             \\
Equal Effort Fairness                                                                                 &                                                      &                                              & $\checkmark$                                &                                                       & $\checkmark$                                            &                                                             \\
\bottomrule
\end{tabular}
}
\label{tab:classification}
\end{table}

In the context of fair machine learning, we use $S \in \{s^+, s^-\}$ to denote the marginalization attribute, $Y \in \{y^+, y^-\}$ to denote the decision, and $\X$ to denote a set of non-marginalization attributes. The underlying mechanism of the population over the space $S\times \X \times Y$ is represented by a causal model $\mathcal{M}$, which is associated with a causal graph $\CG$. Fig. \ref{fig:cgexp} shows a causal graph that will be used to illustrate fairness notions throughout this section. With $\mathcal{M}$, we want to reason about counterfactual queries, e.g., ``what would the prediction have been for this individual if their marginalization attribute value changed?'' A historical dataset $\D$ is drawn from the population, which is used to construct a predictor $ h: \X, S \rightarrow \Y $. Note that the input of the predictor can be a subset of $\X, S$ and we use $\widehat{PA}$ to denote the set of input features of the predictor when introducing counterfactual error rate in Section \ref{sec:cer}. The causal model for the population over space $S\times \X \times \Y$ can be considered the same as $\mathcal{M}$, except that function $f_{Y}$ is replaced with a predictor $h$. Most fairness notions involve either $Y$ or $\hat{Y}$ in their counterfactual quantity and, roughly speaking, they correspond to statistical parity. A few fairness notions, e.g., counterfactual direct error rate \cite{zhang2018equality}, correspond to the concept of equalized odds and involve both $Y$ and $\hat{Y}$ in their counterfactual quantity. We also mark them in Table \ref{tab:classification}.

For each fairness notion, we have two versions: strict and relaxed. The strict version means there is absolutely no discrimination effect, whereas the relaxed version often compares the causal effect with $\tau$, a use-defined threshold for discrimination. In our discussions, we adhere to the strict version when introducing each fairness notion. 

Since the majority of causal-based fairness notions are defined in terms of the non-observable quantities of interventions and counterfactuals, their applicability depends heavily on the identifiability of those quantities from observational data. We refer readers who are interested in learning the specifics of identifiability theory and criteria, and how they can be used to decide the applicability of causal-based fairness metrics to \cite{makhlouf2021survey}. In this section, we simply present causal-based fairness notions and discuss their relationships. 
\begin{figure}[h!]
    \centering
    \includegraphics[scale=.5]{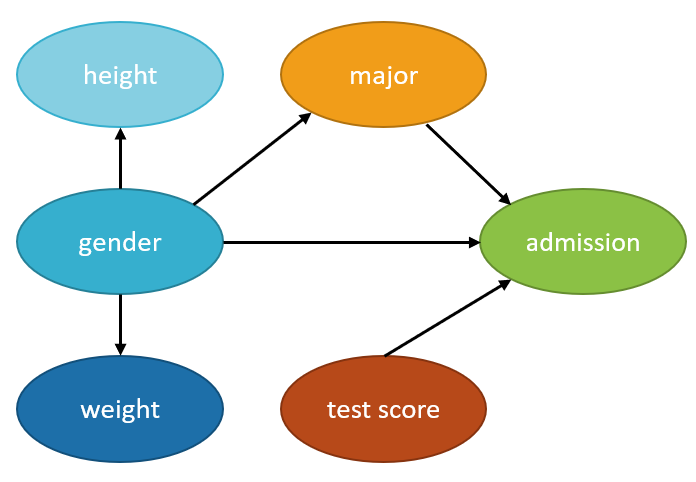}
    \caption{Causal graph of the college admission example used throughout this section.}
    \label{fig:cgexp}
\end{figure}

\subsection{Total, Natural Direct, and Natural Indirect Causal Fairness}
\label{sec:cf}

Discrimination can be viewed as the causal effect of $S$ on $Y$. Total causal fairness answers the question of if the marginalization attribute $S$ changed (e.g., changing from marginalized group $s^{-}$ to non-marginalized group $s^{+}$), how would the outcome $Y$ change on average? A straightforward strategy to answer this question is to measure the average causal effect of $S$ on $Y$ when $S$ changes from $s^{-}$ to $s^{+}$, an approach called total causal fairness.

\begin{definition} [Total Causal Fairness]\label{def:tcf}
Given the marginalization attribute $S$ and decision $Y$, we achieve total causal fairness if \[\TCE(s_1, s_0) = P(y_{s_1}) - P(y_{s_0}) = 0\] where $s_1, s_0 \in \{ s^+, s^-\}$.
\end{definition}

In fact, the causal effect of $S$ on $Y$ includes the direct discriminatory effect, the indirect discriminatory effect, and the explainable effect. 
In \cite{pearl2013direct}, Pearl proposed the use of NDE and NIE to measure the direct and indirect discrimination. Recall that $\mathrm{NDE}(s_1, s_0) = P(y_{s_1, \mathbf{Z}_{s_0}}) - P(y_{s_0})$ and $\mathrm{NIE}(s_1, s_0) = P(y_{s_0, \mathbf{Z}_{s_1}}) - P(y_{s_0})$ where $\mathbf{Z}$ is the set of mediator variables. When applied to the example in Figure \ref{fig:cgexp}, the mediator variable could be the {\em major}. $P(y_{s_1, \mathbf{Z}_{s_0}})$ in NDE is the probability of $Y=y$ had $S$ been $s_1$ and had $\mathbf{Z}$ been the value it would naturally take if $S=s_0$. Similarly, NIE measures the indirect effect of $S$ on $Y$. However, NIE does not distinguish between explainable and indirect discrimination. 

\subsection{Path-Specific Causal Fairness}
\label{sec:pscf}
In \cite{zhang2017causal}, Zhang et al. introduced path-specific causal fairness based on path-specific causal effect \cite{pearl2009causality}. Different from total, natural direct, and natural indirect causal effects, the path-specific causal effect is based on graph properties of the causal graph, and characterizes the causal effect in term of specific paths. 

\begin{definition} [Path-Specific Causal Fairness]\label{def:pscf}
Given the marginalization attribute $S$, decision $Y$, and redlining attributes $\mathbf{R}$, define $\pi_{d}$ as the path set that contains some paths from $S$ to $Y$. We achieve path-specific causal fairness if 
\[\mathrm{PE}_{\pi}(s_1,s_0) = P(y_{s_1 \vert \pi, s_0 \vert \bar{\pi}}) - P(s_{x_0}) =0\] 
where $s_1, s_0 \in \{ s^+, s^-\}$. Specifically, define $\pi_{d}$ as the path set that contains only $S\rightarrow Y$ and define $\pi_{i}$ as the path set that contains all the causal paths from $S$ to $Y$ which pass through some redlining attributes of $\mathbf{R}$. We achieve direct causal fairness if $\mathrm{PE}_{\pi_{d}}(s_1,s_0)=0$, and indirect causal fairness if $\mathrm{PE}_{\pi_{i}}(s_1,s_0)=0$.
\end{definition}

Direct discrimination considers the causal effect transmitted along the direct path from $S$ to $Y$, i.e., $S\rightarrow Y$.
The physical meaning of $\mathit{PE}_{\pi_{d}}(s^{+},s^{-})$ can be explained as the expected change in decisions of individuals from marginalized group $s^{-}$, if the decision makers are told that these individuals were from the non-marginalized group $s^{+}$. When applied to the example in Fig. \ref{fig:cgexp}, it means that the expected change in admission of applications is actually from the marginalized group (e.g., female), when the admission office is instructed to treat the applicants as from the non-marginalized group (e.g., male). 

Indirect discrimination considers the causal effect transmitted along all the indirect paths from $S$ to $Y$ that contain the redlining attributes. The physical meaning of $\mathit{PE}_{\pi_{i}}(s^{+},s^{-})$ is the expected change in decisions of individuals from marginalized group $s^{-}$, if the values of the redlining attributes in the profiles of these individuals were changed as if they were from the non-marginalized group $s^{+}$. When applied to the example in Fig. \ref{fig:cgexp}, it means the expected change in admission of the marginalized group if they had the same gender makeups shown in the major as the non-marginalized group. 

The following propositions \cite{zhang2017causal} further show two properties of the path-specific effect metrics.

\begin{proposition}
If path set $\pi$ contains all causal paths from $S$ to $Y$ and $S$ has no parent in $\mathcal{G}$, then we have
\begin{equation*}
\mathrm{PE}_{\pi}(s_1,s_0) = \mathrm{TE}(s_1,s_0) = P(y^{+}|s_1)-P(y^{+}|s_0).
\end{equation*}
\end{proposition}

$P(y^{+}|s^{+})-P(y^{+}|s^{-})$ is known as the \emph{risk difference}. Therefore, the path-specific effect metrics can be considered as an extension to the risk difference for explicitly distinguishing the discriminatory effects of direct and indirect discrimination from the total causal effect.

\begin{proposition}
For any path sets $\pi_{d}$ and $\pi_{i}$, we do not necessarily have \[\mathrm{PE}_{\pi_{d}}(s_1,s_0)+\mathrm{PE}_{\pi_{i}}(s_1,s_0)=\mathrm{PE}_{\pi_{d}\cup \pi_{i}}(s_1,s_0).\]
\end{proposition}
This implies that there might not be a linear connection between direct and indirect discrimination.

\subsection{Counterfactual Fairness}
In Section \ref{sec:cf} and \ref{sec:pscf}, the intervention is performed on the whole population. These metrics deal with effects on an entire population, or on the average individual from a population. But, up to this point we are still missing the ability to talk about ``personalized causation" at the level of particular events of individuals \cite{PearlMackenzie18}. Counterfactuals provide such a tool. If we infer the post-intervention distribution while conditioning on certain individuals, or groups specified by a subset of observed variables, the inferred quantity will involve two worlds simultaneously: the real world represented by causal model $\mathcal{M},$ as well as the counterfactual world $\mathcal{M}_x$. Such causal inference problems are called counterfactual inference, and the distribution of $Y_{x}$ conditioning on the real world observation $\mathbf{O}=\mathbf{o}$ is denoted by $P(y_{x}| \mathbf{o})$.

In \cite{kusner2017counterfactual}, Kusner et al. defined counterfactual fairness to be the case where the outcome would have remained the same had the marginalization attribute of an individual or a group been different, and all other attributes been equal. 

\begin{definition}[Counterfactual Fairness]\label{def:cf}
	Given a factual condition $\VO = \vo$ where $\VO \subseteq \{S, \X, Y \}$, we achieve counterfactual fairness if 
	\[\CE(s_1, s_0| \vo)  = P(y_{s_1} | \mathbf{o}) - P(y_{s_0} | \mathbf{o}) =0\]
	where $s_1, s_0 \in \{ s^+, s^-\}$.
\end{definition}

Note that we can simply define a classifier as counterfactually fair by replacing outcome $Y$ with the predictor $\hat{Y}$ in the above equation. The meaning of counterfactual fairness can be interpreted as follows when applied to the example in Figure \ref{fig:cgexp}. Applicants are applying for admission and a predictive model is used to make the decision $\Y$. We concern ourselves with an individual from marginalized group $s^-$ who is specified by a profile $\vo$. The probability of the individual to get a positive decision is $P(\y|s^-,\vo)$, which is equivalent to $P(\y_{s^-}|s^-,\vo)$ since the intervention makes no change to $S$'s value of that individual. Now assume the value of $S$ for the individual had been changed from $s^-$ to $s^{+}$. The probability of the individual to get a positive decision after the hypothetical change is given by $P(\y_{s^+}|s^-, \vo)$. Therefore, if the two probabilities $P(\y_{s^-}|s^-, \vo)$ and $P(\y_{s^+}|s^-, \vo)$ are identical, we can claim the individual is treated fairly as if they had been from the other group.

\subsection{Counterfactual Effects}
\label{sec:ce}

In \cite{zhang2018fairness}, Zhang and Bareinboim introduced three fine-grained measures of the transmission of change from stimulus to effect called the counterfactual direct, indirect, and spurious effects. Throughout Section \ref{sec:ce}, we use $\mathbf{W}$ to denote all the observed intermediate variables between $S$ and $Y$ and use the group with $S=s_0$ as the baseline to measure changes of the outcome. 

\begin{definition}[Counterfactual Direct Effect]
	\label{def:ctf-de}
	Given a SCM, the counterfactual direct effect (Ctf-DE) of intervention $S=s_1$ on $Y$ (with baseline $s_0$) conditioned on $S=s$ is defined as  \[\textrm{Ctf-DE}_{s_0,s_1}(y|s) = P(y_{s_1,\mathbf{W_{s_0}}}|s) - P(y_{s_0}|s).\] 
\end{definition}

$Y_{s_1,\mathbf{W}_{s_0}} =y |S=s$ is a more involved counterfactual compared to NDE and can be read as ``the value of $Y$ would be had $S$ been $s_1$, while $\mathbf{W}$ is kept at the same value that it would have attained had $S$ been $s_0$, given that $S$ was actually equal to $s$.'' 

\begin{definition}[Counterfactual Indirect Effect]
	\label{def:ctf-ie}
	Given a SCM, the counterfactual indirect effect (Ctf-IE) of intervention $S=s_1$ on $Y$ (with baseline $s_0$) conditioned on $S=s$ is defined as  \[\textrm{Ctf-IE}_{s_0,s_1}(y|s) = P(y_{s_0,\mathbf{W}_{s_1}}|s) - P(y_{s_0}|s).\] 
\end{definition}

Ctf-IE measures changes in the probability of the outcome $Y$ would be $y$ had $S$ been $s_0$, while changing $\mathbf{W}$ to whatever level it would have naturally obtained had $S$ been $s_1$, in particular, for the individuals that $S=s_0$.

\begin{definition}[Counterfactual Spurious Effect ]
	\label{def:ctf-se}
	Given a SCM, the counterfactual spurious effect (Ctf-SE) of $S=s_1$ on $Y=y$ (with baseline $s_0$) is defined as \[\textrm{Ctf-SE}_{s_0,s_1}(y) = P(y_{s_0}|s_1) - P(y|{s_0}).\] 
\end{definition}

$\text{Ctf-SE}_{s_0,s_1}(y)$ measures the difference in the outcome $Y=y$ had $S$ been $s_0$ for the individuals that would naturally choose $S$ to be $s_0$ versus $s_1$.  

\begin{proposition}
For a SCM, if $S$ has no direct (indirect) causal path connecting $Y$ in the causal graph, then $\textrm{Ctf-DE}_{s_0,s_1}(y|s)=0$ ($\textrm{Ctf-IE}_{s_0,s_1}(y|s)=0$) for any $s$, $y$; if $S$ has no back-door path connecting $Y$ in the causal graph, then $\textrm{Ctf-SE}_{s_0,s_1}(y) = 0$ for any $y$. 
\end{proposition}

Building on these measures, the Zhang and Bareinboim derived the causal explanation formula for the disparities observed in total variation. Recall that the total variation is simply the difference between the conditional distributions of $Y$ when observing $S$ changing from $s_0$ to $s_1$. 

\begin{definition}[Total Variation] \label{def:tv}
	The total variation (TV) of $S=s_1$ on $Y=y$ (with baseline $s_0$) is given by
	\[ \mathrm{TV}_{s_0, s_1}(y) = P(y|s_1) - P(y|s_0). \]
\end{definition}

\begin{theorem}[Causal Explanation Formula] \label{th:expf}
For any $s_0$, $s_1$, $y$, the total variation, counterfactual spurious, direct, and indirect effects obey the following relationship: 
 \[\mathrm{TV}_{s_0,s_1}(y) = \textrm{Ctf-SE}_{s_0,s_1}(y) + \textrm{Ctf-IE}_{s_0,s_1}(y|s_1) - \textrm{Ctf-DE}_{s_1,s_0}(y|s_1),\] 
 \[\mathrm{TV}_{s_0,s_1}(y) = \textrm{Ctf-DE}_{s_0,s_1}(y|s_0) - \textrm{Ctf-SE}_{s_1,s_0}(y) - \textrm{Ctf-IE}_{s_1,s_0}(y|s_0).\]
 \end{theorem}

Theorem \ref{th:expf} allows the machine learning designer to quantitatively evaluate fairness and explain the total observed disparity of decisions through different discriminatory mechanisms. For example, the first formula shows that the total disparity experienced by the individuals who have naturally attained $s_1$ (relative to $s_0$) is equal to the disparity associated with spurious discrimination, plus the advantage it lost due to indirect discrimination, minus the advantage it would have gained without direct discrimination. 

\subsection{Path-Specific Counterfactual Fairness}
In \cite{wu2019pcfairness}, Wu et al. proposed path-specific counterfactual fairness (PC fairness) that covers the previously mentioned causality-based fairness notions. Letting $\Pi$ be all causal paths from $S$ to $Y$ in the causal graph, the path-specific counterfactual fairness metric is defined as follows.

\begin{definition}[Path-specific Counterfactual Fairness (PC Fairness)] \label{def:pscf}
Given a factual condition $\VO = \vo$ where $\VO \subseteq \{S, \X, Y \}$ and a causal path set $\pi$, we achieve the PC fairness if 
	\[\PCE_{\pi}(s_1, s_0| \vo) = P(y_{s_1 \vert \pi, s_0 \vert \bar{\pi}}|\mathbf{o}) - P(y_{s_0}|\mathbf{o}) =0\]
where $s_1, s_0 \in \{ s^+, s^-\}$.
\end{definition}

\begin{table}\small
\centering
\setlength{\extrarowheight}{0pt}
\addtolength{\extrarowheight}{\aboverulesep}
\addtolength{\extrarowheight}{\belowrulesep}
\setlength{\aboverulesep}{0pt}
\setlength{\belowrulesep}{0pt}
\caption{Connection between path-specific counterfactual fairness (PC Fairness) and other fairness notions.}
\begin{tabular}{ll} 
\toprule
\rowcolor[rgb]{0.592,0.792,0.792} \multicolumn{1}{c}{\textcolor[rgb]{0,0.502,0.502}{\textbf{Description}}} & \multicolumn{1}{c}{\textcolor[rgb]{0,0.502,0.502}{\textbf{Relating to PC Fairness}}} \\ 
\hline
Total Causal Fairness                                                                                      & $\mathbf{O}=\emptyset$~and~$\pi = \Pi$                                                \\
Direct Causal Fairness                                                                                     & $\mathbf{O}=\emptyset$~and~$\pi=\pi_{d}=\{S\rightarrow \Y\}$                          \\
Indirect Causal Fairness                                                                                   & $\mathbf{O}=\emptyset$~and~$\pi=\pi_{i}\subset\Pi$                                    \\
Counterfactual Fairness                                                                                    & $\mathbf{O}=\{S,\mathbf{X}\}$~and~$\pi=\Pi$                                           \\
Counterfactual Direct Effect (Ctf-DE)                                                                      & $\mathbf{O}=\{S,Y\}$~and~$\pi=\pi_{d}$                                                \\
Counterfactual Indirect Effect (Ctf-IE)~                                                                   & $\mathbf{O}=\{S,Y\}$~and~$\pi_{i}$                                                    \\
\bottomrule
\end{tabular}
\label{tab:connection}
\end{table}

We point out that we can simply define the PC Fairness on a classifier by replacing outcome $Y$ with the predictor $\hat{Y}$ in the above equation. Previous causality-based fairness notions can be expressed as special cases of the PC fairness
based on the value of $\mathbf{O}$ (e.g., $\emptyset$ or $S,{\mathbf{X}}$) and the value of $\pi$ (e.g., $\Pi$ or $\pi_d$). Their connections are summarised in Table~\ref{tab:connection}, where $\pi_d$ contains the direct edge from $S$ to $\Y$, and $\pi_i$ is a path set that contains all causal paths passing through any redlining attributes (i.e., a set of attributes in $\mathbf{X}$ that cannot be legally justified if used in decision-making). The notion of PC fairness also resolves new types of fairness, e.g., individual indirect fairness, which means discrimination along the indirect paths for a particular individual. Formally, the individual indirect fairness can be directly defined and analyzed using PC fairness by letting $\mathbf{O}=\{S,\mathbf{X}\}$ and $\pi=\pi_{i}$.

\subsection{Proxy Fairness}
\label{sec:proxy}
In \cite{kilbertus2017avoiding}, Kilbertus et al. proposed proxy fairness. A proxy is a descendant of $S$ in the causal graph whose observable quantity is significantly correlated with $S$, but should not affect the prediction.  

\begin{definition}[Proxy Discrimination]
	\label{def:proxy}
	A predictor $\hat{Y}$ exhibits no proxy discrimination based on a proxy $P$ if for all $p,p' \in \text{Dom}(P, P(\hat{Y}_p) = P(\hat{Y}_{p'}))$. 
\end{definition}

Note that $P(\hat{Y}_p)$ is equivalent to $P(\hat{Y}|do(P=p)$. Intuitively, a predictor satisfies proxy fairness if the distribution of $\hat{Y}$ under two interventional regimes in which $P$ set to $p$ and $p'$ is the same. The authors presented the conditions and developed procedures to remove proxy discrimination given the structural equation model. 

\subsection{Justifiable Fairness}
In \cite{salimi2019interventional}, Salimi et al. presented a pre-processing approach for removing the effect of any discriminatory causal relationship between the marginalization attribute and classifier predictions by manipulating the training data to be non-discriminatory. The repaired training data can be seen as a sample from a hypothetical fair world.

\begin{definition}[$\mathbf{K}$-fair]
	\label{def:kfair}
For a give set of variables $\mathbf{K}$, a decision function is said to be $\mathbf{K}$-fair with regarding to $S$ if, for any context $\mathbf{K}=\mathbf{k}$ and any outcome $Y=y$, 
$P(y_{s_0, \mathbf{k}}) = P(y_{s_1,\mathbf{k}})$. 
\end{definition}

Note that the notion of $\mathbf{K}$-fair intervenes on both the sensitive attribute $S$ and marginalization attributes $\mathbf{K}$. It is more fine-grained than proxy fairness, but it does not attempt to capture fairness at the individual level. The authors further introduced justifiable fairness for applications where the user can specify admissible variables through which it is permissible for the marginalization attribute to influence the outcome. 

\begin{definition}[Justifiable Fairness]
	\label{def:justifiable}
A fairness application is justifiable fair if it is $\mathbf{K}$-fair with regarding to all supersets $\mathbf{K} \supseteq \mathbf{A}$ where $\mathbf{A}$ is the set of admissible variables. 
\end{definition}

Different from previous causal-based fairness notions, which require the presence of the underlying causal model, the justifiable fairness notion is based solely on the notion of intervention. The user only requires specification of a set of admissible variables and does not need to have a causal graph. The authors also introduced a sufficient condition for testing justifiable fairness that does not require access to the causal graph. However, with the presence of the causal graph, if all directed paths from $S$ to $Y$ go through an admissible attribute in $\mathbf{A}$, then the algorithm is justifiably fair. If the probability distribution is faithful to the causal graph, the converse also holds. 

\subsection{Counterfactual Error Rate}
\label{sec:cer}
Zhang and Bareinboim \cite{zhang2018equality} developed a causal framework to link the disparities realized through equalized odds (EO) and the causal mechanisms by which the marginalization attribute $S$ affects change in the prediction $\hat{Y}$. EO, also referred to as error rate balance, considers both the ground truth outcome $Y$ and predicted outcome $\hat{Y}$. EO achieves fairness through the balance of the misclassification rates (false positive and negative) across different demographic groups. They introduced a family of counterfactual measures that allows one to explain the misclassification disparities in terms of the direct, indirect, and spurious paths from $S$ to $\hat{Y}$ on a structural causal model. Different from all previously discussed causal-based fairness notions, counterfactual error rate considers both $Y$ and $\hat{Y}$ in their counterfactual quantity. 

\begin{definition}[Counterfactual Direct Error Rate]
	\label{def:cder}
	Given a SCM and a classifier $\hat{y}=f(\widehat{pa})$ where $\widehat{PA}$ is a set of input features of the predictor, the counterfactual direct error rate ($\mathrm{ER}^d$) for a sub-population $s,y$ (with prediction $\hat{y} \ne y)$ is defined as \[\mathrm{ER}^d_{s_0,s_1}(\hat{y}|s,y) = P(\hat{y}_{s_1,y,(\widehat{PA}\backslash S)_{s_0,y}}|s,y) - P(\hat{y}_{s_0,y}|s,y).\] 
\end{definition}
For an individual with the marginalization attribute $S=s$ and the true outcome $Y=y$, the counterfactual direct error rate calculates the difference of two terms. The first term is the prediction $\hat{Y}$ had $S$ been $s_1$, while keeping all the other features $\widehat{PA}\backslash S$ at the level that they would attain had $S=s_0$ and $Y=y$, whereas the second term is the prediction $\hat{Y}$ the individual would receive had $S$ been $s_0$ and $Y$ been $y$. 

\begin{definition}[Counterfactual Indirect Error Rate]
	\label{def:cier}
	Given a SCM and a classifier $\hat{y}=f(\widehat{pa})$, the counterfactual indirect error rate ($\mathrm{ER}^i$) for a sub-population $s,y$ (with prediction $\hat{y} \ne y)$ is defined as \[\mathrm{ER}^i_{s_0,s_1}(\hat{y}|s,y) = P(\hat{y}_{s_0,y,(\hat{PA}\backslash S)_{s_1,y}}|s,y) - P(\hat{y}_{s_0,y}|s,y).\] 
\end{definition}

\begin{definition}[Counterfactual Spurious Error Rate]
	\label{def:cser}
	Given a SCM and a classifier $\hat{y}=f(\hat{pa})$, the counterfactual spurious error rate ($\mathrm{ER}^s$) for a sub-population $s,y$ (with prediction $\hat{y} \ne y)$ is defined as \[\mathrm{ER}^s_{s_0,s_1}(\hat{y}|y) = P(\hat{y}_{s_0,y}|s_1,y) - P(\hat{y}_{s_0,y}|s_0,y).\] 
\end{definition}

The counterfactual spurious error rate can be read as ``for two demographics $s_0$, $s_1$ with the same true outcome $Y=y$, how would the prediction $\hat{Y}$ differ had they both been $s_0$, $y$?''

Building on these measures, Zhang and Bareinboim \cite{zhang2018equality} derived the causal explanation formula for the error rate balance. The equalized odds notion constrains the classification algorithm such that its disparate error rate is equal to zero across different demographics. 

\begin{definition}[Error Rate Balance] \label{def:erb}
	The error rate (ER) balance is given by
	\[ \mathrm{ER}_{s_0, s_1}(\hat{y}|y) = P(\hat{y}|s_1,y) - P(\hat{y}|s_0,y). \]
\end{definition}

\begin{theorem}[Causal Explanation Formula of Equalized Odds] \label{th:expeo}
For any $s_0$, $s_1$, $\hat{y}$, $y$, we have the following relationship: 
 \[\mathrm{ER}_{s_0,s_1}(\hat{y}|y) = \mathrm{ER}^d_{s_0,s_1}(\hat{y}|s_0,y) - \mathrm{ER}^i_{s_1,s_0}(\hat{y}|s_0,y) - \mathrm{ER}^s_{s_1,s_0}(\hat{y}|y).\] 
\end{theorem}

The above theorem shows that the total disparate error rate can be decomposed into terms, each of which estimates the adverse impact of its corresponding discriminatory mechanism. 

\subsection{Individual Equalized Counterfactual Odds}
In \cite{pfohl2019counterfactual}, Pfohl et al. proposed the notion of individual equalized counterfactual odds that is an extension of counterfactual fairness and equalized odds. The notion is motivated by clinical risk prediction and aims to achieve equal benefit across different demographic groups. 

\begin{definition}[Individual Equalized Counterfactual Odds]\label{def:ieco}
	Given a factual condition $\VO = \vo$ where $\VO \subseteq \{\X, Y \}$, predictor $\Y$ achieves the individual equalized counterfactual odds if 
	\[ P(\hat{y}_{s_1} | \mathbf{o},y_{s_1}, s_0) - P(\hat{y}_{s_0} | \mathbf{o}, y_{s_0}, s_0) =0\]
	where $s_1, s_0 \in \{ s^+, s^-\}$.
\end{definition}

The notion implies that the predictor must be counterfactually fair given the outcome $Y$ matching the counterfactual outcome $y_{s_0}$. Therefore, in addition to requiring predictions to be the same across factual/counterfactual samples, those samples must also share the same value of the actual outcome $Y$. In other words, it considers the desiderata from both counterfactual fairness and equalized odds. 

\subsection{Fair on Average Causal Effect}
In \cite{khademi2019fairness}, Khademi et al. introduced two definitions of group fairness: fair on average causal effect (FACE), and fair on average causal effect on the treated (FACT) based on the Rubin-Neyman potential outcomes framework. Let $Y_i(s)$ be the potential outcome of an individual data point $i$ had $S$ been $s$. 

\begin{definition}[Fair on Average Causal Effect (FACE)]
	\label{def:face}
A decision function is said to be fair, on average over all individuals in the population, with respect to $S$, if $\mathbb{E}[Y_i(s_1) - Y_i(s_0)] =0$.
\end{definition}

FACE considers the average causal effect of the marginalization attribute $S$ on the outcome $Y$ at the population level and is equivalent to the expected value of the $\TCE(s_1, s_0)$ in the structural causal model.

\begin{definition}[Fair on Average Causal Effect on the Treated (FACT)]
	\label{def:fact}
A decision function is said to be fair with respect to $S$, on average over individuals with the same value of $s_1$, if $\mathbb{E}[Y_i(s_1) - Y_i(s_0)|S_i =s_1] =0$.
\end{definition}
 
FACT focuses on the same effect at the group level. This is equivalent to the expected value of $ETT_{s_1,s_0}(Y)$. The authors used inverse probability weighting to estimate FACE and use matching methods to estimate FACT.  

\subsection{Equality of Effort}
In \cite{DBLP:conf/www/HuangW0W20}, Huang et al. developed a fairness notation called equality of effort. When applied to the example in Fig. \ref{fig:cgexp}, we have a dataset with $N$ individuals with attributes $(S, T, \mathbf{X}, Y)$ where $S$ denotes the marginalization attribute {\em gender} with domain values $\{ s^+, s^-\}$, $Y$ denotes a decision attribute {\em admission} with domain values $\{ y^+, y^-\}$, $T$ denotes a legitimate attribute such as {\em test score}, and $\mathbf{X}$ denotes a set of covariates. For an individual $i$ in the dataset with profile $(s_{i}, t_{i}, \mathbf{x}_{i}, y_{i})$, they may ask the counterfactual question, how much they should improve their test score such that the probability of their admission is above a threshold $\gamma$ (e.g., $80\%$).

\begin{definition}[$\gamma$-Minimum Effort]
	\label{def:min_effort}
	For individual $i$ with value $(s_{i}, t_{i}, \mathbf{x}_{i}, y_{i})$, the minimum value of the treatment variable to achieve $\gamma$-level outcome is defined as:
	\[
	\Psi_i (\gamma) = \argmin_{t\in T} \big\{ \mathbb{E}[Y_i(t)] \geq \gamma)    \}
	\]
	and the minimum effort to achieve $\gamma$-level outcome is $\Psi_i (\gamma)- t_{i}$.
\end{definition}

If the minimal change for individual $i$ has no difference from that of counterparts (individuals with similar profiles except the marginalization attribute), individual $i$ achieves fairness in terms of equality of effort. As $Y_i(t)$ cannot be directly observed, we can find a subset of users, denoted as $I$, each of whom has the same (or similar) characteristics ($\mathbf{x}$ and $t$) as individual $i$. $I^*$ denotes the subgroup of users in $I$ with the marginalization attribute value $s^*$ where $* \in \{+,-\}$ and $\mathbb{E}[Y_{I^*}(t)]$ denotes the expected outcome under treatment $t$ for the subgroup $I^*$.

\begin{definition}[$\gamma$-Equal Effort Fairness]
	\label{def:equ_effort_i}
	For a certain outcome level $\gamma$, the equality of effort for individual $i$ is defined as
	\[
	\Psi_{I^+}(\gamma) = \Psi_{I^-}(\gamma).
	\]
where $\Psi_{I^*}(\gamma) = \argmin_{t\in T} \{\mathbb{E}[Y_{I^*}(t)] \geq \gamma \}$ is the minimal effort needed to achieve $\gamma$ level of outcome variable within the subgroup $* \in \{+,-\}$.
\end{definition}

Equal effort fairness can be straightforwardly extended to the system (group) level by replacing $I$ with the whole dataset $D$ (or a particular group). Different from previous fairness notations that mainly focus on the the effect of the marginalization attribute $S$ on the decision attribute $Y$, the equality of effort instead focuses on to what extend the treatment variable $T$ should change to make the individual achieve a certain outcome level. This notation addresses the concerns whether the efforts that would need to make to achieve the same outcome level for individuals from the marginalized group and the efforts from the non-marginalized group are different.

\section{Philosophical and Sociological Criticism of Fair Machine Learning}
As we've noted throughout the paper, there have been several attempts to define fairness quantitatively. Some think that the rapid growth of this new field has led to widely inconsistent motivations, terminology, and notation, presenting a serious challenge for cataloging and comparing definitions \cite{mitchell2021}. Through this article, we try to remedy the fact that there has not been more effort spent on aligning quantitative definitions with philosophical measures. After all, what weight does a quantitative definition hold when not grounded in humanistic values?

We continue that discussion in this section. Despite much headway in the last several years, our work in fair machine learning is far from over. In fact, there are several missteps that the field has taken that we need to remedy before truly being able to call our methods ``fair.'' These issues include rigid categorization strategies, improper terminology, hurtful assumptions and abstractions, and issues of diversity and power struggles.

\subsection{Rigid ``Box-Like'' Categorization}
In most of the publications we discussed, fairness is enforced on rigidly structured groups or categories. For instance, many papers consider the binary categories of \textit{male or female} and \textit{White or Black} as the main axis along which to determine if an algorithm is fair or not. These ontological\footnote{Ontology is a branch of philosophy that is concerned with concepts such as existence, being, becoming, and reality \cite{Hu2020}.} assumptions, though helpful in simplifying the problem at hand, are often misplaced. 

The problem with narrowing concepts like gender and race down to simple binary groups is that there is no precise definition of what a ``group'' or ``category'' is in the first place. Despite not having a common interpretation, it is widely accepted in the social sciences that groups and categories are social constructs, not rigid boxes into which a person can be placed. \textit{Constructionist ontology} is the understanding that socially salient categories such as gender and race are not embodied by sharing a physical trait or genealogical features, but are in fact constituted by a web of social relations and meanings \cite{Hu2020}. Social construction does not mean that these groups are not real, but that these categories of race and gender are brought into existence and shaped into what we know them to be by historical events, social forces, political power, and/or colonial conquest \cite{Hacking1999}. When we treat these social constructs as rigidly defined attributes, rather than structural, institutional, and relational circumstances, it minimizes the structural aspects of algorithmic (un)fairness \cite{Hanna2020}. The very concept of fairness itself can only be understood when framed in the viewpoint of the specific social group being considered. 

Specific to racial categorization, Sebastian Benthall and critical race scholar Bruce D. Haynes discuss how ``racial classification is embedded in state institutions, and reinforced in civil society in ways that are relevant to the design of machine learning systems'' \cite{Benthall2019}. Race is widely acknowledge in the social science field to be a social construction tied to a specific context and point of history, rather than to a certain phenotypical property. Hanna et al. explain that the meaning of a ``race'' at any given point in time is tied to a specific \textit{racial project} which is an explanation and interpretation of racial identities according to the efforts in organizing and distributing resources along particular racial lines \cite{Hanna2020}. They express that it would be more accurate to describe race as having relational qualities, with dimensions that are symbolic or based on phenotype, but that are also contingent on specific social and historical contexts \cite{Hanna2020}.

The fair machine learning community needs to understand the multidimensional aspects of concepts such as race and gender and need to seriously consider the impact that our conceptualization and operationalization of historically marginalized groups have on these groups today when defining what a ``group'' or ``category'' is in a specific fair machine learning setting. ``To oversimplify is to do violence, or even more, to re-inscribe violence on communities that already experience structural violence'' \cite{Hanna2020}. The simplifications we make erase the social, economic, and political complexities of racial, gender, and sexuality categories. These counterfactual-based methodologies tend to treat groups as interchangeable, obscuring the unique oppression encountered by each group \cite{Hu2020}. Overall, we cannot do meaningful work in fair machine learning without first understanding and specifying the social ontology of the human groupings about which we are concerned will be the basis for unfairness \cite{Hu2020}.

\subsection{Unintentionally Adverse Terminology}
It is natural to take words such as ``bias'' and ``protected groups'' at face value when reading a fair machine learning publication. Especially when we, as technically minded researchers, would rather spend more time understanding the functionality of an algorithm, rather than the schematics of a particular word. But ``placation is an absolution'' \cite{Hampton2021} and ``Language shapes our thoughts" \cite{PearlMackenzie18}. Throughout this work (and many of the works mentioned within) the term \textit{algorithmic bias} is used liberally and without much thought. However, the word ``bias'' actively removes responsibility from the algorithm or dataset creator by obscuring the social structures and byproducts of oppressive institutions that contribute to the output of the algorithm \cite{Hampton2021}. It makes the effect of bias (i.e., an unfair model) out to be purely accidental. 

So why use ``bias'' then? Mainly because the word oppression is strong and polarizing \cite{Hampton2021, Frye-1983}\footnote{And despite being proposed originally by Safiya Noble in 2018 \cite{Noble2018}, the term never caught on.}. \textit{Algorithmic oppression} as a theoretical concept acknowledges that there are systems of oppression that cannot simply be reformed, and that not every societal problem has (or should have) a technological solution. Algorithmic oppression analyzes the ways that technology has violent impacts on marginalized peoples' lives, and in doing so it does not water down the impact to ``discrimination'' or ``implicit bias'' because doing so fundamentally invalidates the struggles and hardships that oppressed people endure \cite{Hampton2021}.

In addition to \textit{Oppression over Bias}, Hampton also comments on the term ``protected groups''. They note that calling marginalized groups like Black, LGBTQIA+, or even females, ``protected groups'' is a ``meaningless gesture, although well intentioned'' \cite{Hampton2021}. This is because, in reality, these groups are not protected, but oppressed and disparaged and calling them ``protected groups'' does nothing to change their circumstances. 

We echo the sentiments of Hampton \cite{Hampton2021}. This section is more of a critique of our language than a request to overhaul an already confusing field in terms of terminology. Let it serve as a reminder that our choice of words have very real consequences beyond simply explaining the techniques of our method. 

\subsection{Damaging Assumptions and Abstractions}

\subsubsection{Assumptions}
When designing a fair machine learning model, many elements are generally assumed and not distinctly specified. Some of these assumption include the societal objective hoped to be fulfilled by deploying a fair model, the set of individuals subjected to classification by the fair model, and the decision space available to the decision makers who will interact with the model's final predictions \cite{mitchell2021}. These assumptions can have undesirable consequences when they do not hold in the actual usage of the model. Each assumption is a choice that fundamentally determines if the model will ultimately advance fairness in society \cite{mitchell2021}. Additionally, it rarely is the case that the moral assumptions beneath the fairness metrics are explained \cite{hertweck2021}.

Of particular importance is the assumption of the population who will be acted upon by the model, i.e., the individuals who will be subjected to classification by the model. The way that a person comes to belong in a social category or grouping may reflect underlying (objectionable) social structures, e.g., the ``predictive'' policing that targets racial minorities for arrest \cite{mitchell2021}. A model that satisfies fairness criteria when evaluated only on the population to which the model is applied may overlook unfairness in the process by which individuals came to be subject to the model in the first place \cite{mitchell2021}.

Starting with clearly articulated goals can improve both fairness and accountability. Recent criticisms of fair machine learning have rightly pointed out that quantitative notions of fairness can restrict our thinking when we aim to make adjustments to a decision-making process, rather than to address the societal problems at hand. While algorithmic thinking runs such risks, quantitative approaches can also force us to make our assumptions more explicit and clarify what we are treating as background conditions. In doing so, we have the opportunity to be more deliberate and have meaningful debate about the difficult policy issues that we might otherwise hand-wave away, such as: ``what is our objective'', and ``how do we want to go about achieving it'' \cite{mitchell2021} ?

\subsubsection{Abstractions}
Abstraction is one of the cornerstones of computing. It allows a programmer to hide all but the needed information about an object to reduce complexity and increase efficiency. But, abstraction can also lead to the erasure of critical social and historical contexts in problems where fair machine learning is necessary \cite{Hanna2020}. Almost all of the proposed fair machine learning metrics bound the surrounding system tightly to only consider the machine learning model, the inputs, and the outputs, while completely abstracting away any social context \cite{selbst2019abstraction}. By abstracting away the social context in which fair machine learning algorithms are deployed, we no longer are able to understand the broader context that determines how fair our outcome truly is. 

Selbst et. al call these abstraction pitfalls \textit{traps} -- failure modes that occur when failing to properly understand and account for the interactions between a technical system and our humanistic, societal, world \cite{selbst2019abstraction}. Specifically, they name five specific traps that arise when we fail to consider how the social concept aligns with technology and we recall them below: 

\begin{enumerate}
    \item Framing Trap: failure to model the entire system over which a social criterion, such as fairness, will be enforced.
    \item Portability Trap: failure to understand how re-purposing algorithmic solutions designed for one social context may be misleading, inaccurate, or otherwise do harm when applied to a different context.
    \item Formalism Trap: failure to account for the full meaning of social concepts such as fairness, which can be procedural, contextual, contestable, and cannot be resolved through mathematical formalism.
    \item Ripple Effect Trap: failure to understand how the insertion of technology into an existing social system changes the behaviors and embedded values of the pre-existing system.
    \item Solutionism Trap: failure to recognize the possibility that the best solution to a problem may not involve technology. 
\end{enumerate}

Selbst et al.'s main proposed solution is to focus on the process of determining where and how to apply technical solutions, and when applying technical solutions causes more harm than good \cite{selbst2019abstraction}. They point out that in order to come to such a conclusion, technical researchers will need to either learn new social science skills or partner with social scientists on projects. Additionally, they point out that we must also become more comfortable with going against the intrinsic nature of the computer scientist to use abstraction, and be at ease with the difficult or unresolvable tensions between the usefulness and dangers of abstraction \cite{selbst2019abstraction}.

\subsection{Causal Specific Problems}
There has been consensus among some in the fair machine learning community that there are limitations when using statistical fairness definitions as they cannot show the connections and causal relationships of the data. Because of this, some researchers have turned to using causal-based methods for explaining the fairness of a machine learning algorithm as they are able to present a formal model that depicts how the relevant features are causally related. Additionally, causal-based methods require that the researchers not only be explicit about their theory of what causal relations generate the observed data, but also their theory of ontological boundaries between different categories and groups \cite{Hu2020}. 

But, despite the explicitness, the theories themselves can be incoherent to what social categories are, which in turn leads to misleading results when applied in high-stake domains \cite{kasirzadeh2021use}. ``To counterfactually suppose a social attribute of an individual requires first specifying what the social categories are and what it means to suppose a different version of an individual with the counterfactually manipulated social property'' \cite{kasirzadeh2021use}. As we discussed above, it is not clear or agreed upon how to define such social categories in the first place. 

Objections to using race or gender as a causal variable are often framed in the terms of not being manipulable, i.e., ``no causation without manipulation'' \cite{holland_1986}. In other words, the physical impossibility of manipulating race or gender should exclude it from being able to be used as a causal variable \cite{Hanna2020}. But, founders of the causal method field, such as Pearl and Kusner, argue that not being able to actually manipulate gender or race in the real world should not stop us from being able to reason about it theoretically \cite{pearl2009causality, kusner2017counterfactual}. Many disagree with this line of thinking though, stating that the role social categories like race or gender play in structuring life experiences makes it illogical to say two individuals are exactly the same, save for their gender or race \cite{kohler-hausmann_2017}. Put simply, in order to talk about the causal effect of social categories, we first need to satisfy what these categories are \cite{kasirzadeh2021use}.

One solution to the above problems with causal models has been proposed by Hu and Kohler-Hausmann - an approach they termed \textit{constitutive models} \cite{Hu2020}. They suggest that formal diagrams of constitutive relations would allow a new line of reasoning about discrimination as they offer a model of how the meaning of a social group is formed from its constitutive features. Constitutive relations show how societal practices, beliefs, regularities, and relations make up a category \cite{Hu2020}. They also note that causal diagrams can simply be reformatted to be a constitutive one, and that because a constitutive model provides a model of what makes a category, it presents entirely the information needed to debate about what practices are discriminatory \cite{Hu2020}.

\subsection{Power Dynamics and Diversity}
Here, we consider three important power dynamics: who is doing the classifying, who is picking the objective function, and who gets to define what counts as science. Starting with the first - who has the power to classify - J. Khadijah Abdurahman says that "it is not just that classification systems are inaccurate or biased, it is who has the power to classify, to determine the repercussions / policies associated thereof, and their relation to historical and accumulated injustice'' \cite{j-khadijah-abdurahman_2019}. As mentioned above, since there is no agreed upon definition of what a group/category is, it is ultimately up to those in power to classify people according to the task at hand. Often, this results in rigid classifications that do not align with how people would classify themselves. Additionally, because of data limitations, most often those in power employ the categories provided by the U.S. census or other taxonomies which stem from bureaucratic processes. But, it is well studied that these categories are unstable, contingent, and rooted in racial inequality \cite{Hanna2020}. When we undertake the process of classifying people, we need to understand what the larger implications of classifying are, and how they further impact or reinforce hurtful social structures.

The second question - who chooses the final optimization function to use in a fair machine learning algorithm - seems fairly intuitive. Of course, those creating fair machine learning methods do. But, should we have this power? The choice of how to construct the objective function of an algorithm is intimately connected with the political economy question of who has ownership and control rights over data and algorithms \cite{kasy2021}. It is important to keep in mind that our work is, overall, for the benefit of marginalized populaces. That being the case, ``it is not only irresponsible to force our ideas of what communities need, but also violent'' \cite{Hampton2021}. ``Before seeking new design solutions, we [should] look for what is already working at the community level'' and ``honor and uplift traditional, indigenous, and local knowledge and practices'' \cite{costanza-chock_2018}. This may require taking to asking the oppressed groups what their communities need, and what we should keep in mind when constructing the optimization algorithm to better serve them. ``We must emphasize an importance of including all communities, and the voices and ideas of marginalized people must be centered as [they] are the first and hardest hit by algorithmic oppression'' \cite{Hampton2021}.

The final question - who chooses what is defined as science - comes from the study of the interplay of feminism with science and technology. Ruth Hubbard, the first woman to hold a tenured professorship position in biology at Harvard, advocated for the inclusion of other social groups besides White men to be allowed to make scientific contributions as ``whoever gets to define what counts as a scientific problem also gets a powerful role in shaping the picture of the world that results from scientific research'' \cite{hubbard_1988}. For a drastic example, consider R.A. Fisher, who for a long period of time was the worlds leading statistician and practically invented large parts of the subject, was also a eugenicist, and thought that ``those who did not take his word as God-given truth were at best stupid and at worst evil" \cite{PearlMackenzie18}.

Despite calls for diversity in science and technology, there is conflicting views on how to go about doing so. Some say that including marginalized populaces will help gain outside perspectives that will overall aid in creating technology to better suit the people it will eventually be used on \cite{mitchell2021}. Others say that this is actually not the case, and more diversity will not automatically solve algorithmic oppression \cite{Hampton2021}. Sociologist Ruha Benjamin points out that ``having a more diverse team is an inadequate solution to discriminatory design practices that grow out of the interplay of racism and captilatism'' as it shifts responsibility from ``our technologies are harming people'' to ``BIPOC\footnote{Black, Indigenous, People of Color} tokens have to fix it'' \cite{Hampton2021, benjamin2019}. By promoting diversity as a solution to the problem of algorithmic oppression, we ``obscure the fact that there are power imbalances that are deeply embedded in societal systems and institutions'' \cite{Hampton2021}.

Regardless of how to solve the diversity issue, it is agreed upon that it is important to engage with marginalized communities and educate them on what fair machine learning is, and how it affects them. ``We will solve nothing without involving our communities, and we must take care to ensure we do not impose elitist ideas of who can and cannot do science and engineering'' \cite{Hampton2021}. It is our view that we, the fair machine learning community, should be having conversations with BIPOC communities about their thoughts on how we, the fair machine learning community, should solve the diversity issue (as well as thoughts on what they need and actually want from our community), and what we can do to help fix the problems machine learning ultimately created in the first place.

\section{Related Works}
Many works have been published over the last few years on fair machine learning, including a handful of survey papers and textbooks that are closely aligned with this field guide \cite{caton2020fairness, mehrabi2019survey, barocas-hardt-narayanan}. While these survey papers do a good job of explaining the mathematical and algorithmic aspects of mitigating bias and achieving fairness, they often leave out critical discussion of philosophical and legal groundings that are important to make a sociotechnical system rather than just a technical one. Additionally, while works exist that align philosophical \cite{ khan_fairness_2021, heidari_moral_2019, lee_formalising_2021, binns_fairness_2018} and legal \cite{xiang2019legalcompatibility, grgic-hlaca_beyond_nodate, corbett-davies_algorithmic_2017, barocas_big_2016} notions with proposed fairness metrics, they often do not go in depth on the technical, algorithmic, and mathematical foundations which are also needed to make a sociotechnical system. Our work resolves this issue by producing a survey that covers both of these realms to allow for practitioners to understand not only how specific fairness metrics function, but their social science groundings as well. 

\section{Conclusion}
In this field guide, we have attempted to remedy a long standing problem in the fair machine learning field, namely, the abstraction of technical aspects of algorithms with their philosophical, sociological, and legal underpinnings. By explaining the details of popular statistics and causal-based fair machine learning algorithms in both formal and social science terminology, ultimately, we recenter algorithmic fairness as a sociotechnical problem, rather than simply a technical one. We hope that this field guide not only helps machine learning practitioners understand how specific algorithms align with long-held humanistic values, but also that it will spark conversation and collaboration with the social science field to construct better algorithms. 

In addition to explaining the metrics themselves, we also offered a critique on the field of fair machine learning as a whole. We do this specifically by calling upon literature produced by those marginalized and underrepresented in the fair machine learning community as they have view points that are critical to understanding how our work actually impacts and affects the social groups they belong to. When designing a fair machine learning algorithm, or any machine learning algorithm at all, we need to be mindful that our work ultimately impacts people beyond the immediate research community and our research labs. Our work should be centered around eliminating harm through algorithmic oppression, not in being (unintentionally) complicit to the violence against oppressed populaces by machine learning.

We conclude with the following call to action. We, the fair machine learning research community, before releasing fair machine learning methods should be intimately aware of their philosophical, social, and legal ties (as these notions ultimately determine how the final model will be implemented and used), as well as how they will actually affect the marginalized community they propose to protect. It is only in this way that we can actually contribute meaningful and ``fair'' machine learning research.

\section*{Acknowledgement}
This work was supported in part by NSF 1910284, 1920920 and 2137335.

\bibliographystyle{unsrt}  
\bibliography{references,nips19}  
\end{document}